\documentclass[sigconf]{acmart}

\AtBeginDocument{%
	\providecommand\BibTeX{{%
			\normalfont B\kern-0.5em{\scshape i\kern-0.25em b}\kern-0.8em\TeX}}}

\copyrightyear{2022} 
\acmYear{2022} 
\setcopyright{acmcopyright}
\acmConference[KDD '22]{Proceedings of the 28th ACM SIGKDD Conference on Knowledge Discovery and Data Mining}{August 14--18, 2022}{Washington, DC, USA}
\acmBooktitle{Proceedings of the 28th ACM SIGKDD Conference on Knowledge Discovery and Data Mining (KDD '22), August 14--18, 2022, Washington, DC, USA}
\acmPrice{15.00}
\acmDOI{10.1145/3534678.3539391}
\acmISBN{978-1-4503-9385-0/22/08}

\usepackage{amsmath}
\usepackage{mathtools}
\usepackage{amsthm}
\usepackage{bm}
\usepackage{subfigure}
\usepackage{multirow}
\usepackage{graphicx}
\usepackage{algorithm}
\usepackage{algorithmic}
\newtheorem{theorem}{Theorem}[section]

\newtheorem{lemma}[theorem]{Lemma}

\theoremstyle{definition}
\newtheorem{definition}[theorem]{Definition}

\theoremstyle{remark}

\begin{document}
	
	\title{Learning Task-relevant Representations for Generalization via Characteristic Functions of Reward Sequence Distributions}
	
	
	\author{Rui Yang}
	\email{yr0013@mail.ustc.edu.cn}
	\affiliation{%
		\institution{University of Science and Technology of China}
		\country{}
	}
	
	\author{Jie Wang$^{\ast}$}
	\thanks{$\ast$ Corresponding Author}
	\email{jiewangx@ustc.edu.cn}
	\affiliation{%
		\institution{
			Institute of Artificial Intelligence\\
			Hefei Comprehensive National Science Center\\
			University of Science and Technology of China}
		\country{}
	}
	
	\author{Zijie Geng}
	\email{ustcgzj@mail.ustc.edu.cn}
	\affiliation{%
		\institution{University of Science and Technology of China}
		\country{}
	}
	
	\author{Mingxuan Ye}
	\email{mingxuanye@miralab.ai}
	\affiliation{%
		\institution{University of Science and Technology of China}
		\country{}
	}
	
	\author{Shuiwang Ji}
	\email{sji@tamu.edu}
	\affiliation{%
		\institution{Texas A\&M University}
		\city{College Station}
		\state{TX}
		\country{}
	}
	
	\author{Bin Li}
	\email{binli@ustc.edu.cn}
	\affiliation{%
		\institution{University of Science and Technology of China}
		\country{}
	}
	
	\author{Feng Wu}
	\email{fengwu@ustc.edu.cn}
	\affiliation{%
		\institution{University of Science and Technology of China}
		\country{}
	}
	
	\renewcommand{\shortauthors}{Rui Yang et al.}
	

	\begin{abstract}
		Generalization across different environments with the same tasks is critical for successful applications of visual reinforcement learning (RL) in real scenarios. However, visual distractions---which are common in real scenes---from high-dimensional observations can be hurtful to the learned representations in visual RL, thus degrading the performance of generalization. To tackle this problem, we propose a novel approach, namely \textbf{C}haracteristic \textbf{Re}ward \textbf{S}equence \textbf{P}rediction (\textbf{CRESP}), to extract the task-relevant information by learning reward sequence distributions (RSDs), as the reward signals are task-relevant in RL and invariant to visual distractions. Specifically, to effectively capture the task-relevant information via RSDs, CRESP introduces an auxiliary task---that is, predicting the characteristic functions of RSDs---to learn task-relevant representations, because we can well approximate the high-dimensional distributions by leveraging the corresponding characteristic functions. Experiments demonstrate that CRESP significantly improves the performance of generalization on unseen environments, outperforming several state-of-the-arts on DeepMind Control tasks with different visual distractions.
	\end{abstract}

	

	\maketitle
	
	\begin{figure}[t]
		\begin{center}
			\centerline{\includegraphics[scale=0.45]{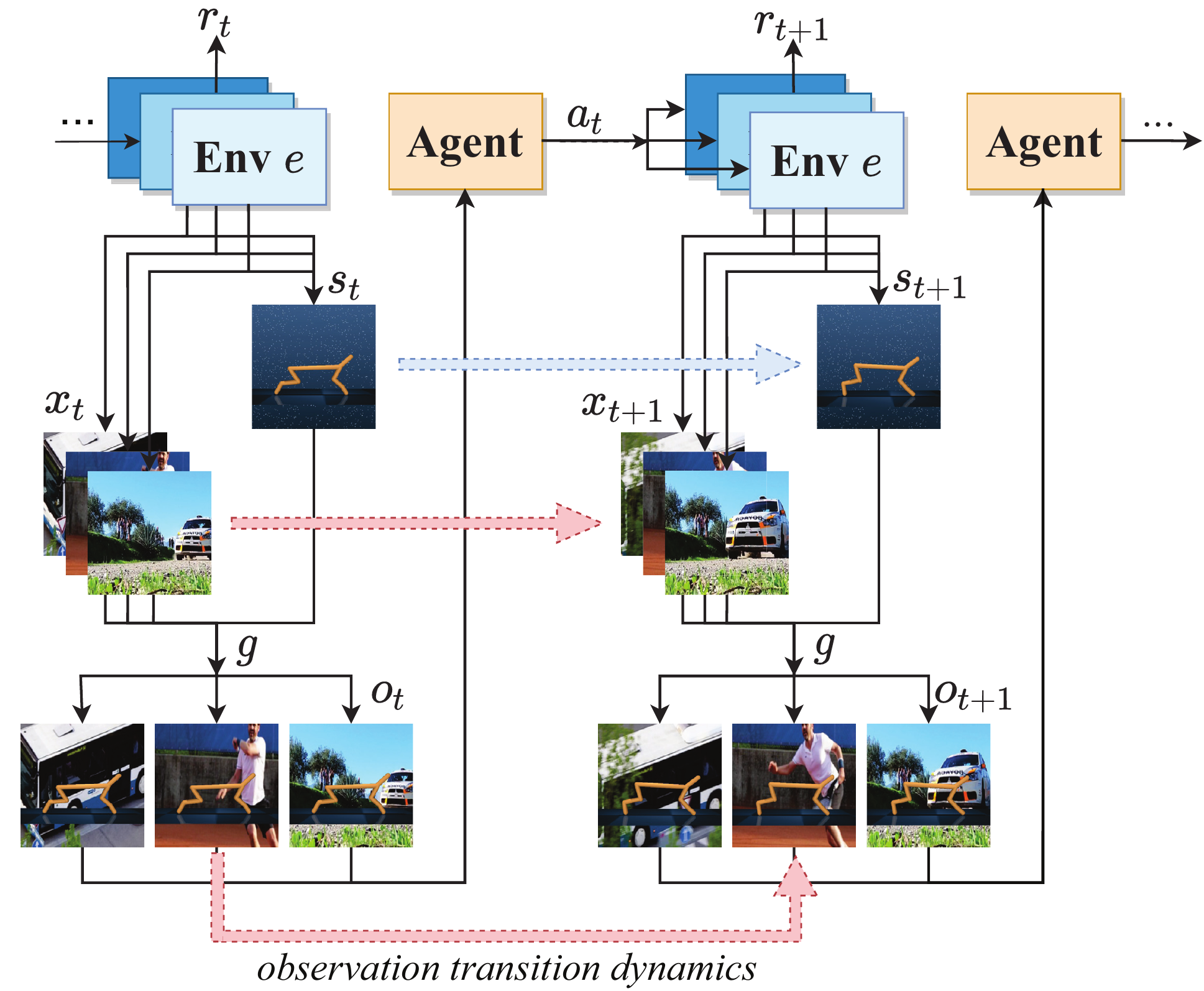}}
			\vskip -0.1in
			\caption{The agent-environment interactions in Block MDPs with visual distractions. Each environment $e$ provides a state $s_t$ and a background $x_t$, which generate an observation $o_t=g(s_t,x_t)$ through a nonlinear function $g$. The agent receives $o_t$ and takes an action $a_t$ in $e$, leading to the transitions of states (from $s_t$ to $s_{t+1}$), backgrounds (from $x_t$ to $x_{t+1}$), and thus the observation transitions (from $o_t$ to $o_{t+1}$). Notice that the red arrows represent the transitions that vary with different environments, while the blue arrow represents the transition invariant to environments.}
			\Description{This figure shows the agent-environment interactions in Block MDP settings with visual distractions.}
			\label{fig-rl}
		\end{center}
		\vskip -0.2in
		\vskip -0.1in
	\end{figure}
	
	\section{Introduction}
	\label{intro}
	
	Visual reinforcement learning (RL) algorithms aim to solve complex control tasks from high-dimensional visual observations. Notable successes include DrQ for locomotion control~\citep{iclr/YaratsKF21}, IMPALA for multi-task learning~\citep{icml/EspeholtSMSMWDF18}, and QT-Opt for robot grasping~\citep{corr/abs-1806-10293}. \mbox{Although} these methods perform well on training environments, they can hardly generalize to new environments, even these environ-ments are semantically similar to the training environments.
	This is because image observations often involve many task-irrelevant visual factors, such as dynamic backgrounds and colors of the object under control. Minor changes in such visual factors may cause large distributional shifts of the environments, which prevent the agent from extracting underlying task-relevant information when we put it into a new environment. This indicates that many existing RL agents memorize the trajectories on specific environments~\citep{nips/RajeswaranLTK17,iclr/SongJTDN20}, rather than learning transferable skills.
	
	To learn a policy with transferable skills for generalization, many prior works focus on learning representations that encode only the task-relevant information while discarding task-irrelevant \mbox{visual} factors.
	Some of them propose similarity metrics~\cite{iandc/LarsenS91,aaai/Castro20} to find semantically equivalent observations for representation learning~\citep{iclr/0001MCGL21,iclr/AgarwalMCB21}.
	Others design objectives by integrating MDP properties to learn a causal representation that is invariant to irrelevant \mbox{features}~\citep{abs-2106-00808,icml/0001LSFKPGP20}.
	These aforementioned methods leverage rewards and transition dynamics to capture task-relevant features. However, the observation transition dynamics (see Figure~\ref{fig-rl}) may induce the task-irrelevant information relating to visual distractions into the representations, thus hindering generalization~\citep{icml/0001LSFKPGP20, l4dc/SonarPM21}. Detailed discussions are in Section~\ref{subsec-4.1}.
	
	In contrast to the above methods, we propose a novel approach, namely \textbf{C}haracteristic \textbf{Re}ward \textbf{S}equence \textbf{P}rediction (CRESP), which only uses reward signals but observation transition dynamics to learn task-relevant representations, as the reward signals are task-relevant in RL and invariant to visual factors.
	To preserve information that is relevant to the task, CRESP introduces the \textit{reward sequence distributions} (RSDs), which are the conditional distributions of reward sequences given a starting \mbox{observation} and various subsequent actions.
	CRESP leverages RSDs to learn a task-relevant representation that only encodes the information of RSDs, which we call \textit{reward sequence representation}.
	Specifically, considering that the characteristic function can specify high-dimensional distributions~\citep{ansari2020characteristic}, we propose to learn such task-relevant representation by an auxiliary task that predicts the \textit{characteristic functions} of RSDs.
	Moreover, we provide a theoretical analysis of the value bounds between the true optimal value functions and the optimal value functions on top of the reward sequence representations.
	Experiments on DeepMind Control Suite~\citep{corr/abs-1801-00690} with visual distractors~\citep{corr/abs-2101-02722} demonstrate that CRESP significantly improves several state-of-the-arts on unseen environments.
	
	Our main contributions in this paper are as follows:
	\begin{itemize}
		\vskip -0.0in
		\vskip -0.0in
		\item We introduce the reward sequence distributions (RSDs) to discard the task-irrelevant features and preserve the task-relevant features.
		\item We propose CRESP, a novel approach that extracts the task-relevant information by learning the characteristic functions of RSDs for representation learning.
		\item Experiments demonstrate that the representations learned by CRESP preserve more task-relevant features than prior methods, outperforming several state-of-the-arts on the majority of tasks by substantial margins.
	\end{itemize}
	
	\section{Related Work}
	\paragraph{Generalization in visual RL}
	The study of generalization in deep RL focuses on the capability of RL methods to generalize to unseen environments under a limited set of training environments.
	Several works propose to apply regularization techniques originally developed for supervised learning, including dropout~\citep{nips/IglCLTZDH19} and batch normalization~\citep{farebrother2018generalization, nips/IglCLTZDH19}.
	Although practical and easy to implement, these methods do not exploit any properties of sequential decision-making problems.
	Other approaches for preventing overfitting focus on data augmentation~\citep{ye2020rotation, lee2020network, laskin2020reinforcement, raileanu2021automatic}, which enlarge the available data space and implicitly provide the prior knowledge to the agent.
	Although these methods show promising results in well-designed experimental settings, strong assumptions such as prior knowledge of the testing environments may limit their real applications. 
	In contrast to these methods, we consider a more realistic setting without assuming this prior knowledge of environments.
	
	\paragraph{Representation Learning in visual RL}
	Many prior works focus on representation learning for generalization in visual RL.
	Some of the works~\citep{lange2010deep,lange2012autonomous} use a two-step learning process, which first trains an auto-encoder by using a reconstruction loss for low-dimensional representations, and then uses this representation for policy optimization.
	However, such representations encode all elements from observations, whether they are relevant to the task or not.
	Other works use bisimulation metrics to learn a representation that is invariant to irrelevant visual features~\citep{iclr/0001MCGL21}.
	However, such methods use the transition dynamics, which vary with the environments, leading to the learned representation involving task-irrelevant features of the visual distractions.
	A recent study~\citep{LehnertLF20} leverages the reward prediction for representation learning. However, the representation learning method only considers finite MDPs, which cannot extend to visual RL tasks.
	
	\paragraph{Characteristic Functions of Random Variables}
	Characteristic functions are the Fourier transforms of probability density functions.
	They are well studied in probability theory and can be used to specify high-dimensional distributions.
	This is because two random variables have the same distribution if and only if they have the same characteristic function.
	Some prior works~\cite{ansari2020characteristic,yu2004empirical} use characteristic functions to solve some statistical problems.
	We leverage this tool for a simple and tractable approximation of high-dimensional distributions.
	Our experiments demonstrate that the characteristic functions perform well to specify the distributions of reward sequences in our method.

	\section{Preliminaries}
	\label{sec:pre}
	In visual RL tasks, we deal with high-dimensional image observations, instead of the states as the inputs. We consider a family of environments with the same high-level task but different visual distractions.
	Denote $\mathcal{E}$ as the set of these environments. We model each environment $e\in\mathcal{E}$ as a Block Markov Decision Process (BMDP)~\citep{du2019provably,icml/0001LSFKPGP20}, which is described by a tuple $\mathcal{M}^e=(\mathcal{S}, \mathcal{O}, \mathcal{A},\mathcal{R}, p,p^e,\gamma)$. Here $\mathcal{S}$ is the state space,\, $\mathcal{O}$ is the observation space, $\mathcal{A}$ is the action space,\, $\mathcal{R}$ is the reward space, which we assume to be bounded,\, $p(s',r|s,a)$ is the state transition probability, $p^e(o',r|o,a)$ is the observation transition probability, which varies with environments $e\in\mathcal{E}$, and $\gamma \in [0,1)$ is the discount factor.
	
	At each time step $t$, we suppose that the environment is in a state $S_t$.\footnote{Throughout this paper, we use uppercase letters such as $S_t$ and $O_t$ to denote random variables, and use lowercase letters such as $s_t$ and $o_t$ to denote the corresponding values that the random variables take.}
	The agent, instead of directly achieving $S_t$, obtains an observa-tion $O_t$ on environment $e\in\mathcal{E}$.
	It is reasonable to assume that the observation is determined by the state and some task-irrelevant visual factors that vary with environments, such as backgrounds or agent colors in DeepMind Control tasks. Symbolically, let $\mathcal{X}$ be the set of such visual factors.
	We suppose that there exists an \textit{observation function} $g:\mathcal{S}\times\mathcal{X}\to\mathcal{O}$ \citep{iclr/SongJTDN20,du2019provably} such that $O_t=g(S_t,X_t)$, where $X_t$ is a random variable in $\mathcal{X}$, independent with $S_t$ and $A_t$, with a transition probability $q^e(x'|x)$.
	See  Figure~\ref{fig-rl} for an illustration.
	We aim to find a policy $\pi (\cdot|o_t)$ that maximizes the expected accumulated reward $\mathbb{E}^e\left[\sum_{t=0}^\infty \gamma^t R_t\right]$ simultaneously in all environments $e\in\mathcal{E}$, where $\mathbb{E}^e[\cdot]$ means that the expectation is taken in the environment $e$.
	
	Moreover, we assume that the environments follow a generalized Block structure~\citep{icml/0001LSFKPGP20,du2019provably}. 
	That is, an observation $o\in\mathcal{O}$ uniquely determines its generating state $s$, and the visual factor $x$.
	This \mbox{assumption} implies that the observation function $g(s,x)$ is invertible with respect to both $s$ and $x$.
	For simplicity, we denote $s=[o]_s$ and $x=[o]_x$ as the generating state and visual factor, respectively.
	Furthermore, we have $p^e(o',r|o,a)=p(s',r|s,a)q^e(x'|x)$, where $s=[o]_s, s'=[o']_s$, $x=[o]_x$ and $x'=[o']_x$.

	\section{Representation Learning via Reward Sequence Distributions}

	\label{sec-4}
	An encoder, or a representation, refers to an embedding function $\Phi:\mathcal{O}\to\mathcal{Z}$, which maps the observational inputs onto a latent state representation space $\mathcal{Z}$. 
	Our goal is to find a suitable representation that encodes only task-relevant information and is invariant to visual distractions.
	In Section~\ref{subsec-4.1}, we discuss the notion of task relevance in visual RL, introduce reward sequence distributions (RSDs), and formulate the reward sequence representations for generalization.
	In Section~\ref{subsec-4.2}, we provide a theoretical analysis that reformulates the reward sequence representation via the characteristic functions of RSDs.
	In Section~\ref{subsec-4.3}, we present a practical method, based on the prediction of characteristic functions of RSDs, to learn such a reward sequence representation.

	\subsection{Task-relevant Invariance in Visual RL}
	\label{subsec-4.1}
	The key idea of our approach is to capture the task-relevant infor-mation across different environments from observations, and leverage such information for representation learning to improve the performance of generalization.
	
	Reward signals and transition dynamics are major properties of MDP, which are commonly used for representation learning in visual RL.
	We start with a discussion on the distractions induced by observation transition dynamics.
	In visual RL, we can hardly learn about the state transition dynamics, as the state space is unavailable in practice.
	Instead, many methods learn the observation transition dynamics by a probabilistic dynamics model~\citep{iclr/0001MCGL21,corr/abs-1807-03748,corr/abs-2112-10504}.
	However, the observation transition dynamics are relevant to the visual factors because they comprise the transition dynamics of both states and task-irrelevant visual factors.
	Formally, we have the reward and observation transition dynamics $p^e(o',r|o,a)=p(s',r|s,a)q^e(x'|x)$. 
	This formula shows that the observation transition probability varies with the environment $e\in\mathcal{E}$.
	We present a case in Figure~\ref{fig-rl} to illustrate the observation transition dynamics.
	Therefore, representations that encode information about observation transition dynamics are subject to visual distractions and have difficulty learning transferable skills.

	\begin{figure}
		\begin{center}
			\centerline{\includegraphics[scale=0.25]{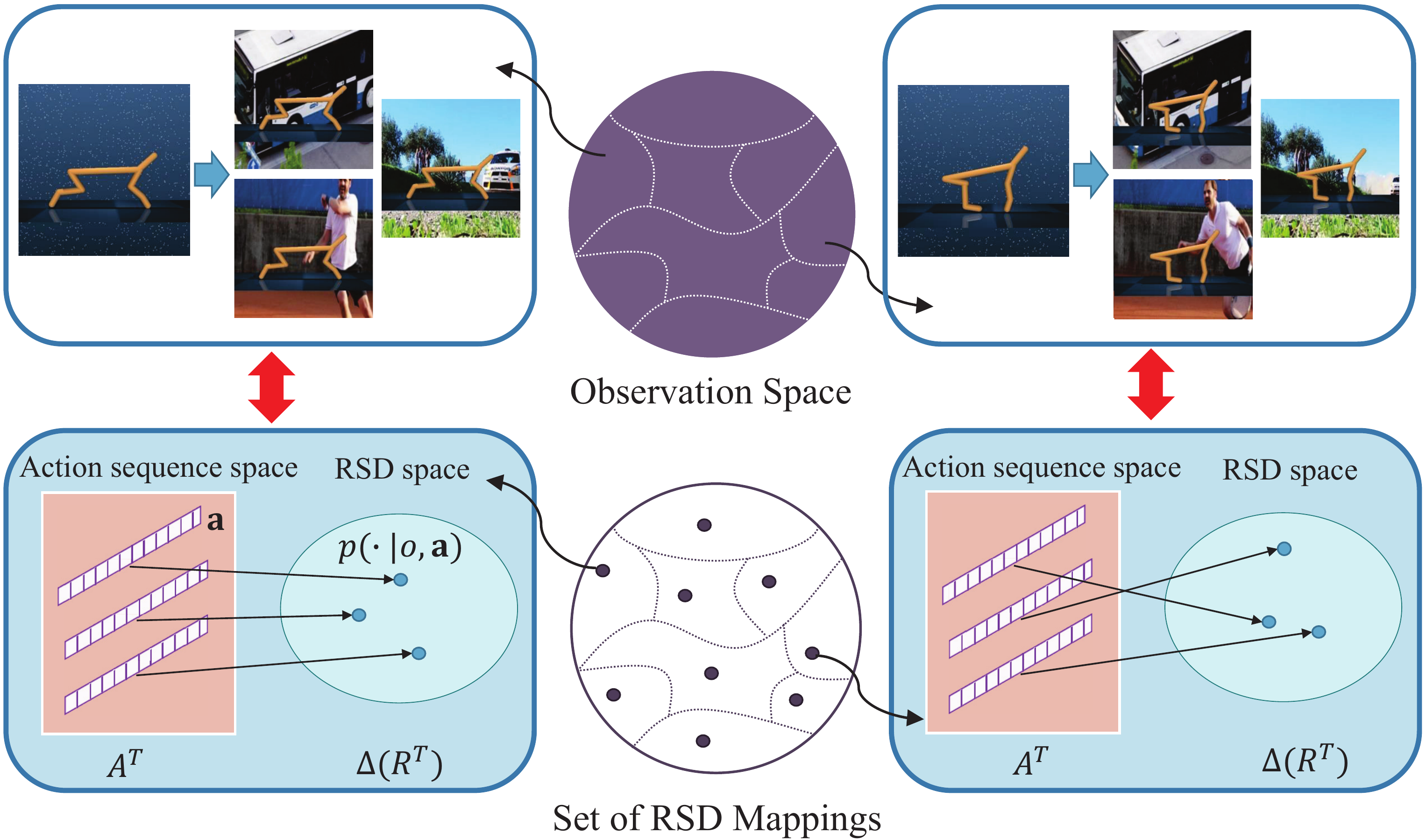}}
			\vskip -0.1in
			\caption{The relationship between observations and RSD mappings.
				We can divide the observation space into different equivalence classes, where
				the equivalent observations are generated from the same state. 
				Each equivalence class corresponds to a same mapping from action sequences $\mathbf{a}\in\mathcal{A}^T$ to reward sequence distributions $p(\cdot|o,\mathbf{a})\in \Delta(\mathcal{R}^T)$.}
			\label{fig-mappings}
			\Description{This figure shows the relationship between the observations and the mappings from action sequences to the corresponding reward sequence distributions.}
		\end{center}
		\vskip -0.1in
		\vskip -0.1in
	\end{figure}
	
	In contrast to observation transition dynamics, the distributions of reward signals are relevant to the RL tasks and are invariant to visual distractions.
	Formally, if two observations $o$ and $o'$ are generated by the same state $s$, i.e., $[o]_s=[o']_s$, then we have $p^e(r|o,a)=p^e(r|o',a)$ for any $a\in\mathcal{A}$ and $e\in\mathcal{E}$.
	This motivates us to use the reward signals instead of observation transition dynamics for representation learning.
	As our goal is to maximize the expected accumulative rewards, what we need is not only the current reward but also the sequences of future rewards.
	Therefore, we propose to utilize the reward sequences for representation learning.
	
	For a mathematical formulation, we introduce some new nota-tions.
	We denote $\mathcal{A}^T=\{\mathbf{a}=(a_1,\cdots, a_T):a_i\in\mathcal{A}\}$ and $\mathcal{R}^T=\{\mathbf{r}=(r_1,\cdots,r_{T}):r_i\in\mathcal{R}\}$ as the spaces of action sequences and reward sequences with length $T$, respectively.
	Let $\Delta (\mathcal{R}^T)$ be the set of probability distributions over $\mathcal{R}^T$.
	At each time step $t$, the sequence of the subsequent actions $\mathbf{A}_t^T=(A_t,\cdots,A_{t+T-1})$ is a $T$-dimensional random vector over $\mathcal{A}^T$. The sequence of the subsequent rewards $\mathbf{R}_{t+1}^T=(R_{t+1},\cdots,R_{t+T})$ is a $T$-dimensional random vector over $\mathcal{R}^T$.
	\footnote{We use bold uppercase letters such as $\mathbf{A}$ and $\mathbf{R}$ to denote random vectors in high-dimensional spaces and use bold lowercase letters such as $\mathbf{a}$ and $\mathbf{r}$ to denote deterministic vectors in such spaces.}

	To clarify our idea, we first consider a deterministic environment.
	Starting from an observation $o_t\in\mathcal{O}$, with the corresponding state $s_t = [o_t]_s\in\mathcal{S}$, suppose that we perform a given action sequence $\mathbf{a}_t^T=(a_t,\cdots,a_{t+T-1})\in\mathcal{A}^T$ and receive a reward sequence $\mathbf{r}_{t+1}^T=(r_{t+1},\cdots,r_{t+T})\in\mathcal{R}^T$ from the environment.
	This reward sequence $\mathbf{r}_{t+1}^T$ is uniquely determined by the starting state $s_t$ and the given action sequence $\mathbf{a}_t^T$.
	Therefore, we can find that the relationship between the given action sequence $\mathbf{a}_t^T$ and the received reward sequence $\mathbf{r}_{t+1}^T$ is invariant to visual distractions.
	We can use such a relationship to identify the task-relevant information from observations.
	To formulate this relationship, we consider the mappings from action sequences $\mathbf{a}\in\mathcal{A}^T$ to the corresponding reward sequences $\mathbf{r}\in\mathcal{R}^T$---that the agent receives from an observation $o$ by following the action sequence $\mathbf{a}$.
	We consider two observations $o$ and $o'$ that have same mappings from $\mathbf{a}\in\mathcal{A}^T$ to $\mathbf{r}\in\mathcal{R}^T$ for any dimension $T$.
	In other words, we suppose that the agent receives the equal reward sequence $\mathbf{r}\in\mathcal{R}^T$ from $o$ and $o'$, when it follows any action sequence $\mathbf{a}\in\mathcal{A}^T$ for any $T$.
	Then the two observations have similar task properties, in the sense that the agent will receive the equal accumulative rewards from $o$ and $o'$ no matter what actions the agent takes.
	Therefore, the mappings from action sequences $\mathbf{a}\in\mathcal{A}^T$ to the corresponding reward sequences $\mathbf{r}\in\mathcal{R}^T$ can be used to identify the task-relevant information from the observations.
	
	We then consider the stochastic environment, the case of which is similar to the deterministic environment. 
	In the stochastic environment, the reward sequence $\mathbf{R}_{t+1}^T$ is random even for fixed observation $o_t$ and action sequence $\mathbf{A}_t^T$.
	Therefore, we cannot simply consider the mappings from $\mathcal{A}^T$ to $\mathcal{R}^T$.
	Instead, we apply the mappings from $\mathcal{A}^T$ to $\Delta(\mathcal{R}^T)$, which map the action sequences to the distributions of the sequences of reward random variables.
	
	Formally, let $p(\mathbf{r}|o,\mathbf{a})$ be the probability density function of the random vector $\mathbf{R}_{t+1}^T$ at the point $\mathbf{r}\in\mathcal{R}^T$, conditioned on the starting observation $O_t=o$ and the action sequence $\mathbf{A}_t^T=\mathbf{a}\in\mathcal{A}^T$.
	For any $o\in\mathcal{O}, \mathbf{a}=(a_1,\cdots,a_T)$, and $\mathbf{r}=(r_2,\cdots,r_{T+1})$, we have
	\begin{align*}
		p(\mathbf{r}|o,\mathbf{a})=p(r_2|s,a_1)p(r_3|s,a_1,a_2)\cdots p(r_{T+1}|s,a_1,\cdots,a_T),
	\end{align*}
	where $s=[o]_s$, and $p(r|s,a_1,\cdots,a_t)$ denotes the probability density function of the reward $r$ that the agent receives, after following an action sequence $(a_1,\cdots,a_t)$, starting from the \mbox{state $s$}.
	Furthermore, for any $o,o'\in\mathcal{O}$ such that $[o]_s=[o']_s$, we have $p(\mathbf{r}|o,\mathbf{a})=p(\mathbf{r}|o',\mathbf{a})$.
	The formulas imply that the conditional distributions $p(\cdot|o,\mathbf{a})$ of reward sequences are determined by the generating states of the observations as well as the action sequences.
	Therefore, the mappings from the action sequences $\mathbf{a}\in\mathcal{A}^T$ to the corresponding RSDs $p(\cdot|o,\mathbf{a})$ are task-relevant and invariant to visual distractions.
	Thus we can use the mappings to determine task relevance.
	See Figure~\ref{fig-mappings} for an illustration.
	
	The analysis above motivates our method that leverages the RSDs to learn representations.
	Specifically, we learn a representation that can derive a function, which maps the action sequences to the corresponding RSDs.
	Formally, we define the $T$-level reward sequence representation as follows.
	\begin{definition}
		A representation $\Phi:\mathcal{O}\to\mathcal{Z}$ is a \textit{$T$-level reward sequence representation} if it can derive the distribution of any reward sequence received from any observation by following any action sequence with length $T$, i.e., there exists $f$ such that
		\begin{align*}
			f(\mathbf{r};\Phi(o),\mathbf{a}) = p(\mathbf{r}|o,\mathbf{a}),\, \forall\, \mathbf{r}\in\mathbf{R}^T,o\in\mathcal{O},\mathbf{a}\in\mathcal{A}^T.
		\end{align*}
	\end{definition}
	
	Intuitively, the $T$-level reward sequence representation encodes the task-relevant information about the relation between the action sequences $\mathbf{a}$ and the RSDs $p(\mathbf{r}|o,\mathbf{a})$ in the next $T$ steps.
	Notice that a $T$-level reward sequence representation is also a $T'$-level reward sequence representation, where $T,T'\in\mathbb{N}^*$ and $T>T'$.
	If $T$ tends to infinity, the representation will encode all task-relevant information from the objective of RL tasks.
	This derives the following definition.
	
	\begin{definition}
		A representation $\Phi:\mathcal{O}\to\mathcal{Z}$ is a \textit{reward sequence representation} if it is a $T$-level reward sequence representation for all $T\in\mathbb{N}^*$.
	\end{definition}
	
	The reward sequence representation is equivalent to a $\infty$-level reward sequence representation.
	In practice, we learn a finite $T$-level reward sequence representation as an approximation of the reward sequence representation.
	To provide a theoretical guarantee for the approximation, the following theorem gives a value bound between the true optimal value function and the value function on top of the $T$-level reward sequence representation.
	
	\begin{theorem}
		\label{thm:bd}
		Let $\Phi:\mathcal{O}\to\mathcal{Z}$ be a $T$-level representation, $V^e_*:\mathcal{O}\to\mathbb{R}$ be the optimal value function in the environment $e\in\mathcal{E}$, $\Bar{V}^e_*:\mathcal{Z}\to\mathbb{R}$ be the optimal value function on the latent representation space, built on top of the representation $\Phi$. Let $\bar{r}$ be a bound of the reward space, i.e., $|r|<\bar{r}$ for any $r\in\mathcal{R}$.
		Then we have
		\begin{align*}
			0\le V^e_*(o)-\Bar{V}^e_*\circ\Phi(o)\le\frac{2\gamma^T}{1-\gamma}\bar{r},
		\end{align*}
		for any $o\in\mathcal{O}$ and $e\in\mathcal{E}$.
	\end{theorem}
	\begin{proof}
		See Appendix~\ref{app-thm:bd}
	\end{proof}

	\subsection{Characteristic Functions for Representation Learning}
	\label{subsec-4.2}
	In Section \ref{subsec-4.1}, we formulate the $T$-level reward sequence representation that can derive the probability density function $p(\mathbf{r}|o,\mathbf{a})$, 
	where $\mathbf{r}\in\mathcal{R}^T$ is a reward sequence, $o\in\mathcal{O}$ is an observation, and $\mathbf{a}\in\mathcal{A}^T$ is an action sequence.
	However, learning the probability density functions is usually technically impractical~\citep{ansari2020characteristic}.
	Leveraging the characteristic function of random vectors, we propose an alternative approach, which is simple to implement and effective to learn the distributions.

	\begin{figure*}
		\begin{center}
			\vskip -0.05in
			\centerline{\includegraphics[scale=0.36]{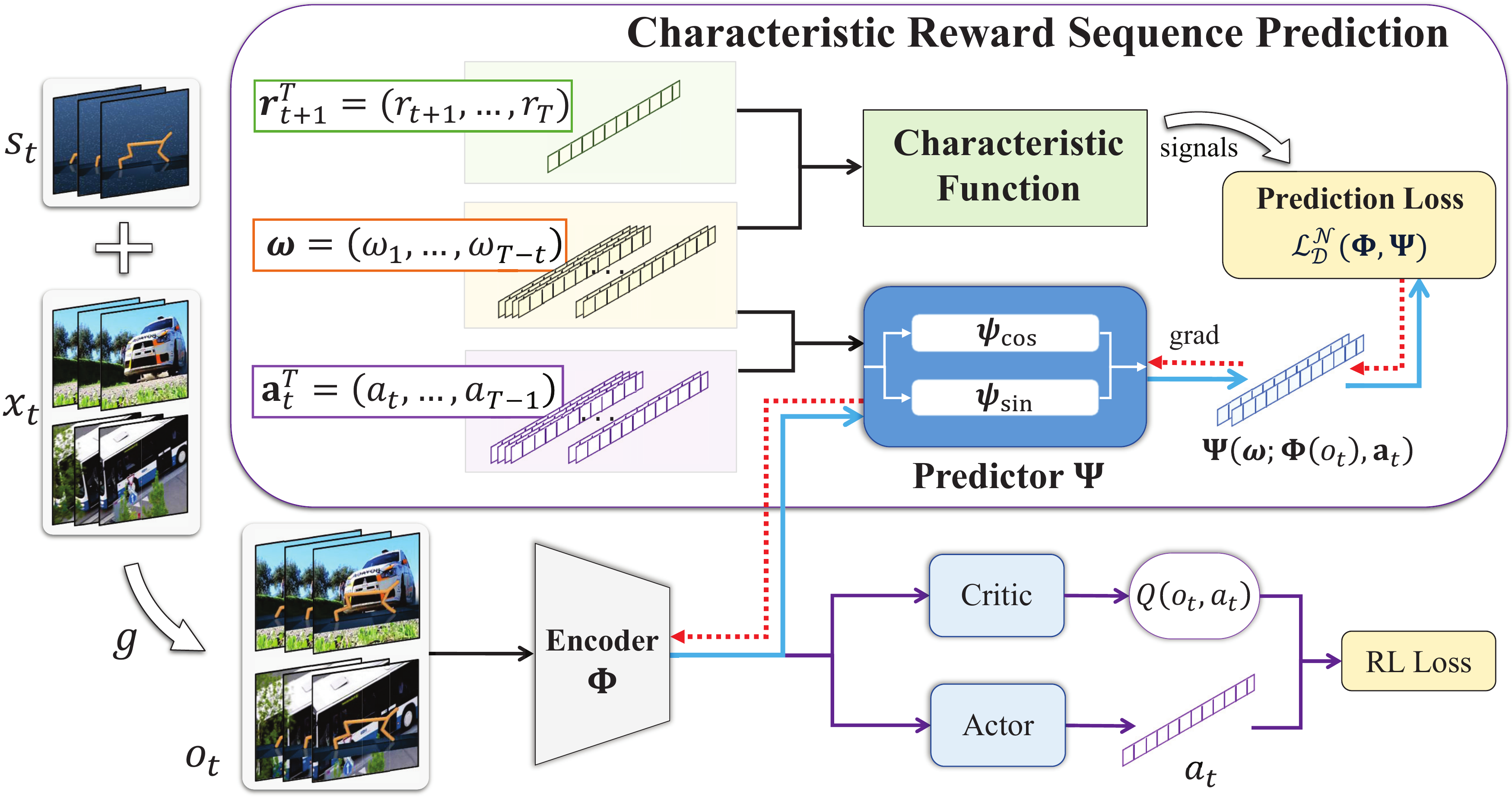}}
			\caption{The overall architecture of CRESP.
				CRESP minimizes the prediction loss to train an encoder $\Phi$, and simultaneously uses $\Phi$ to learn a policy in an actor-critic setting.
				In the prediction task, CRESP predicts the characteristic functions of reward sequence distributions through the encoder $\Phi$ and the predictor $\Psi$ (in the purple box).
				The prediction loss $\mathcal{L^N_D}(\Phi, \Psi)$ provides the gradients (red lines) to update both the predictor $\Psi$ and the encoder $\Phi$.
				Here $\mathbf{r}^T_{t+1}$ and $\mathbf{a}^T_t$ are the sequences drawn from a replay buffer $\mathcal{D}$.
				The inputs $\bm{\omega}$ of characteristic functions are sampled from a Gaussian distribution $\mathcal{N}$.}
			\label{fig-arc}
		\end{center}
		\Description{This figure shows the overall architecture of our approach, CRESP, which predicts the characteristic functions of reward sequence distributions for representation learning, and simultaneously uses the representation for the RL task. }
		\vskip -0.1in
	\end{figure*}

	Consider a random vector $\mathbf{R}$ defined on the space $\mathcal{R}^T$, with a probability density function $p_{\mathbf{R}}(\cdot)$. The characteristic function $\varphi_{\mathbf{R}}:\mathbb{R}^T\to\mathbb{C}$ of $\mathbf{R}$ is defined as
	\begin{align*}
		\varphi_\mathbf{R} (\bm{\omega})=\mathbb{E}_{\mathbf{R}\sim p_{\mathbf{R}}(\cdot)}\left[e^{i\langle\bm{\omega},\mathbf{R}\rangle}\right]=\int e^{i\langle\bm{\omega},\mathbf{r}\rangle}p_{\mathbf{R}}(\mathbf{r})\mathrm{d}\mathbf{r},
	\end{align*}
	where $\bm{\omega}\in\mathbb{R}^T$ denotes the input of $p_{\mathbf{R}}(\cdot)$, and $i=\sqrt{-1}$ is the imaginary unit.
	Since we consider discounted cumulative rewards in RL tasks, we use $\langle\cdot,\cdot\rangle$ to denote the weighted inner product in $\mathbb{R}^T$, i.e., $\langle \bm{\omega}, \mathbf{r}\rangle=\sum_{t=1}^T \gamma^t\omega_tr_t$, where $\gamma$ is the discounted factor.
	
	Characteristic functions are useful tools well studied in probability theory.
	In contrast to the probability density function, the characteristic function has some good basic properties.
	1) $\left|\varphi_{\mathbf{R}}(\bm{\omega})\right|\le\mathbb{E}_{\mathbf{R}\sim p_{\mathbf{R}}(\cdot)}\left|e^{i\langle\bm{\omega},\mathbf{R}\rangle}\right|=1,$ which indicates that the characteristic function always exists and is uniformly bounded.
	2) The characteristic function $\varphi_\mathbf{R}$ is uniformly continuous on $\mathbb{R}^T$, which makes it tractable for learning.
	
	The following lemma states a fact that the distribution of a random vector can be specified by its characteristic function.
	
	\begin{lemma}~\citep{grimmett2020probability}
		\label{lem:cf}
		Two random vectors $\mathbf{X}$ and $\mathbf{Y}$ have the same characteristic function if and only if they have the same probability distribution function.
	\end{lemma}
	
	This lemma implies that we can recapture the information about the distributions of random vectors via their characteristic functions.
	Therefore, instead of learning the conditional density functions of reward sequences that are intractable, we propose to leverage characteristic functions of the RSDs for representation learning.
	Specifically, we have the following theorem.
	\begin{theorem}
		\label{thm:cf}
		A representation $\Phi:\mathcal{O}\to\mathcal{Z}$ is a $T$-level reward sequence representation if and only if there exits a predictor $\Psi$ such that for all $\bm{w}\in\mathbb{R}^{T},o\in\mathcal{O}$ and $\mathbf{a}\in\mathcal{A}^T$,
		\begin{align*}
			\Psi(\bm{\omega};\Phi(o),\mathbf{a})=\varphi_{\mathbf{R}|o,\mathbf{a}}(\bm{\omega})=\mathbb{E}_{\mathbf{R}\sim p(\cdot|o,\mathbf{a})}\left[e^{i\langle\bm{\omega},\mathbf{R}\rangle}\right].
		\end{align*}
	\end{theorem}
	\begin{proof}
		See Appendix~\ref{app-thm:cf}.
	\end{proof}
	Theorem~\ref{thm:cf} provides an equivalent definition of $T$-level reward sequence representation and inspires our novel approach to predict the characteristic functions of RSDs for representation learning.

	\begin{figure*}[h]
		\centering
		\includegraphics[width=4.8cm]{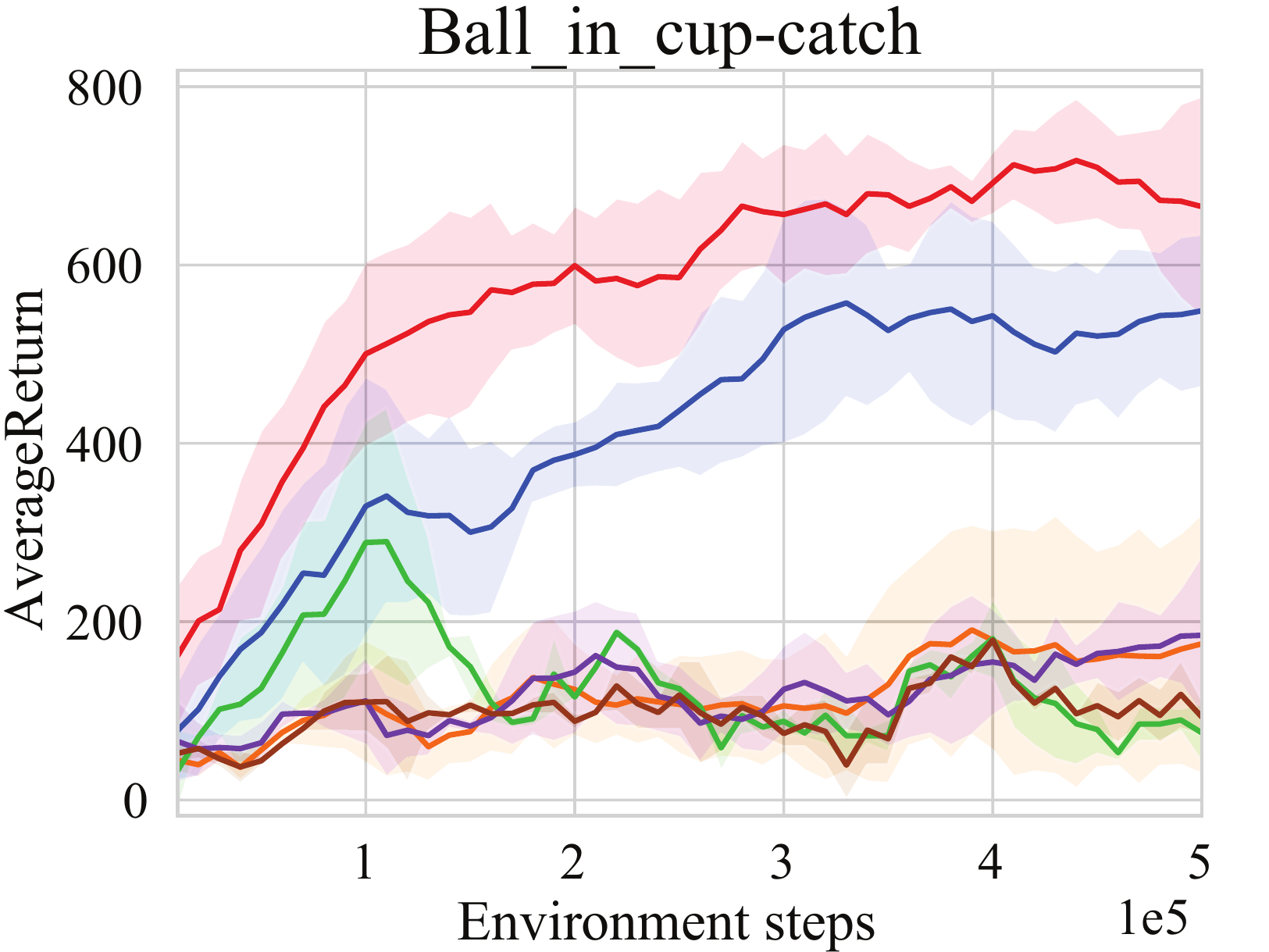}
		\hskip 0.2in
		\includegraphics[width=4.8cm]{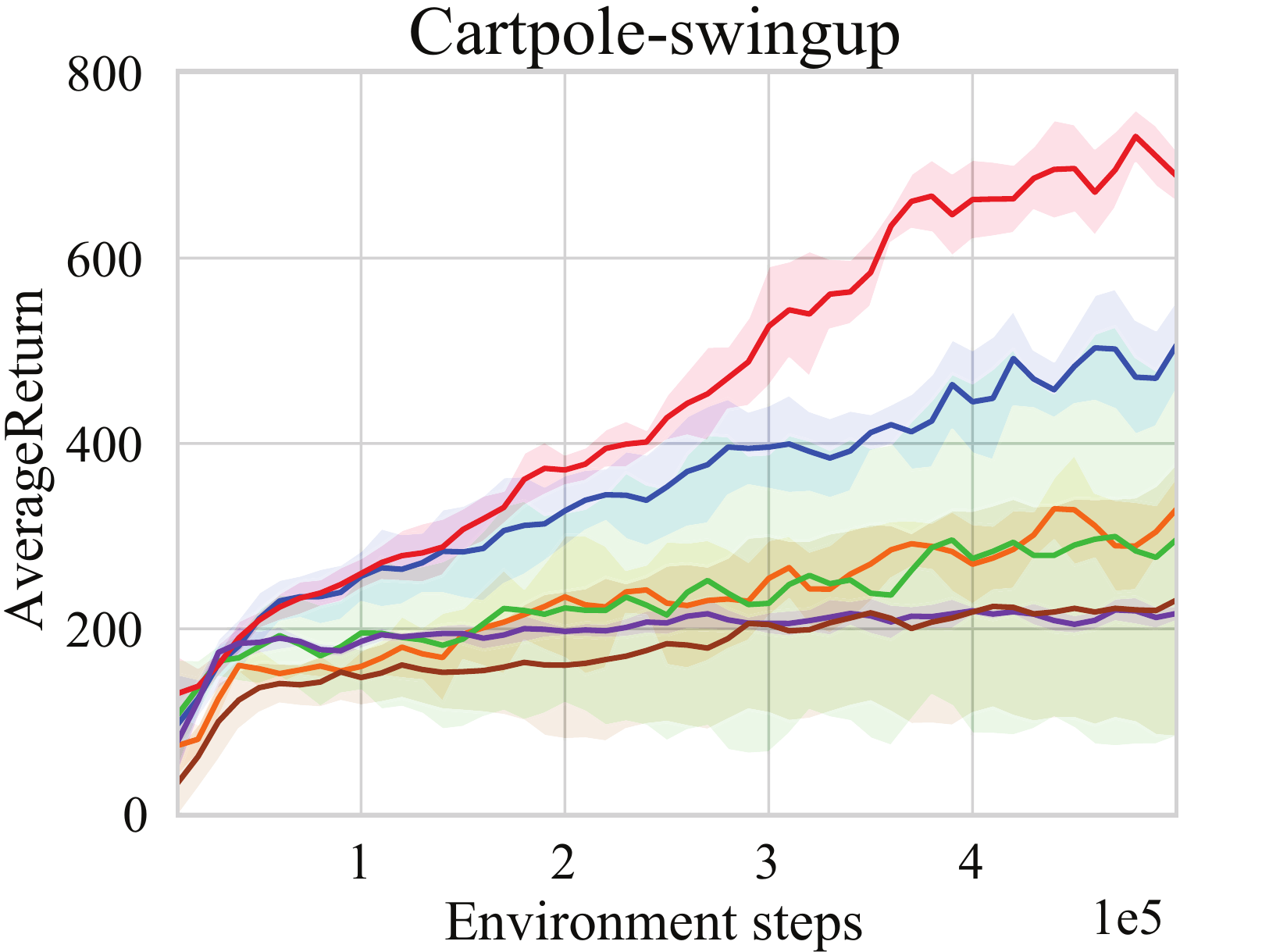}
		\hskip 0.2in
		\includegraphics[width=4.8cm]{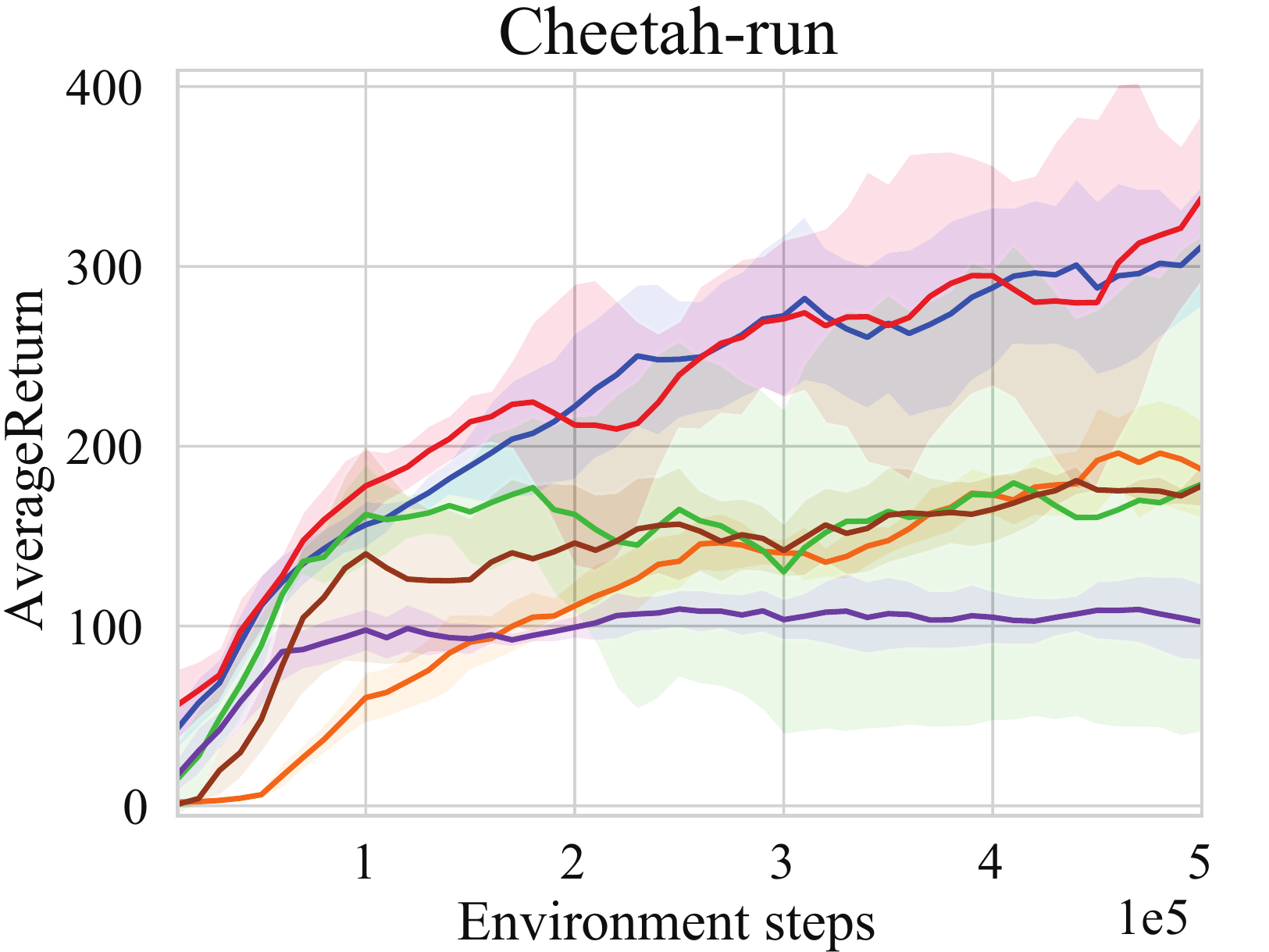}
		
		\vskip 0.02in
		\includegraphics[width=4.8cm]{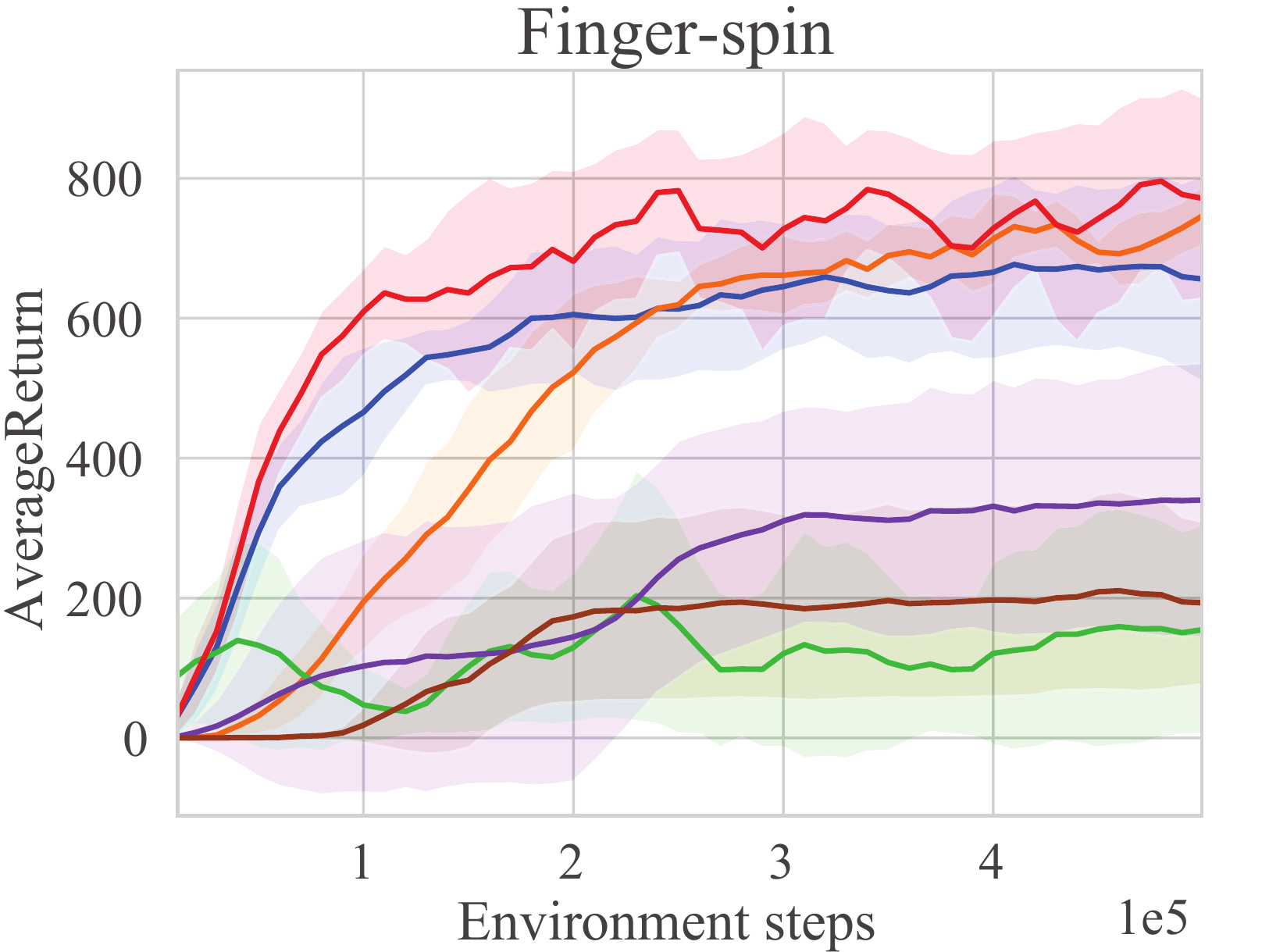}
		\hskip 0.2in
		\includegraphics[width=4.8cm]{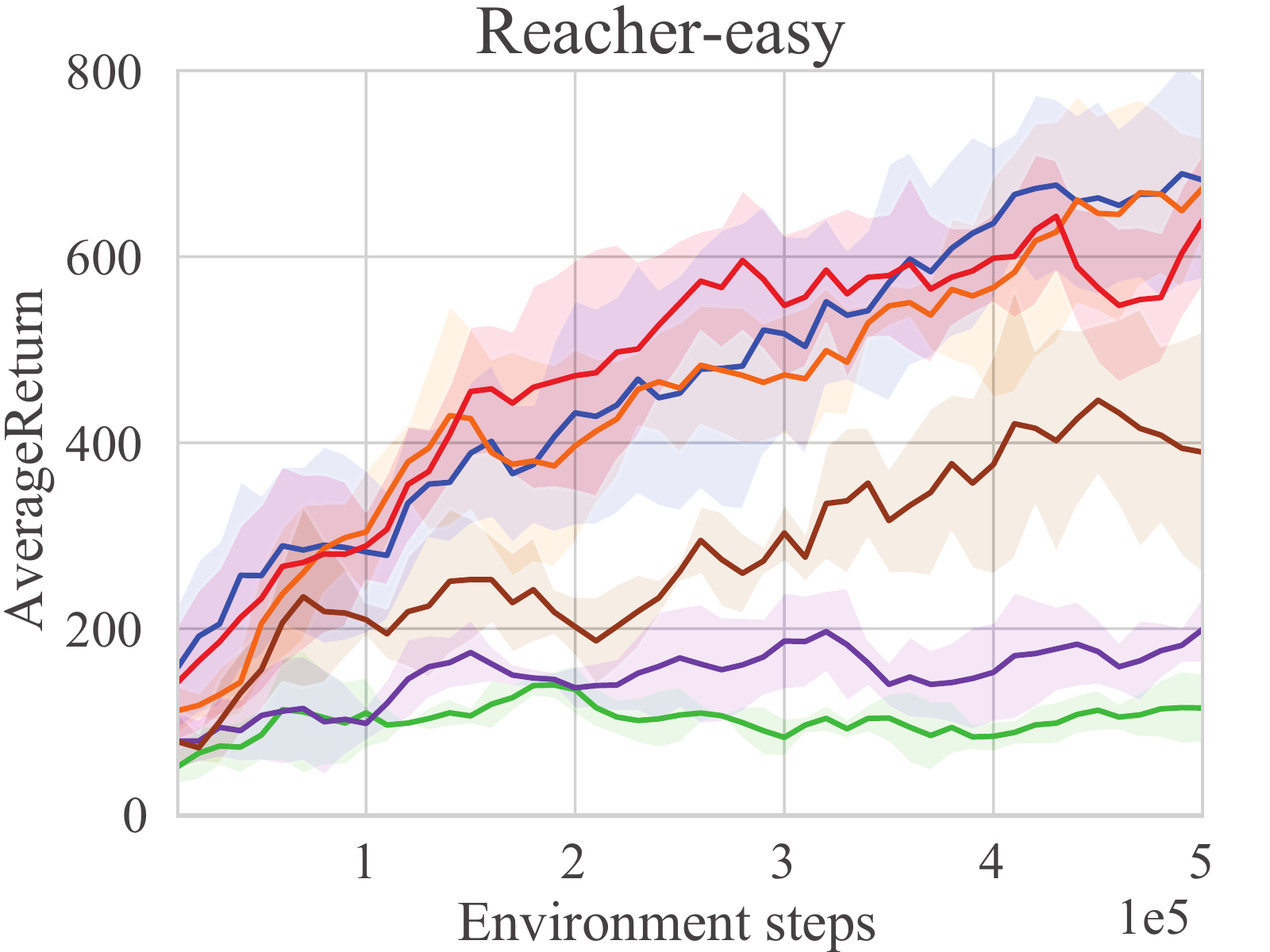}
		\hskip 0.2in
		\includegraphics[width=4.8cm]{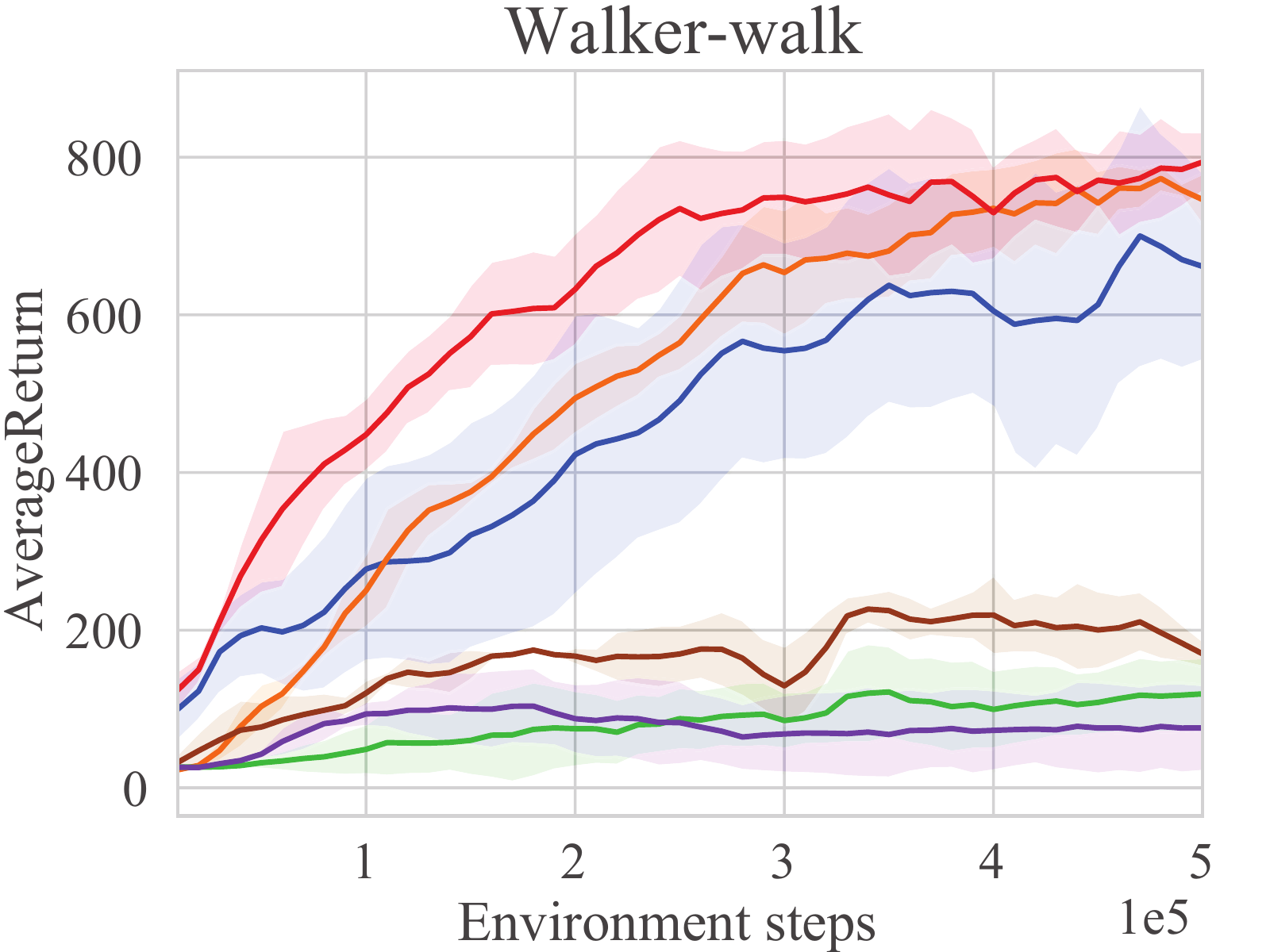}
		
		\vskip 0.02in
		\includegraphics[width=13cm]{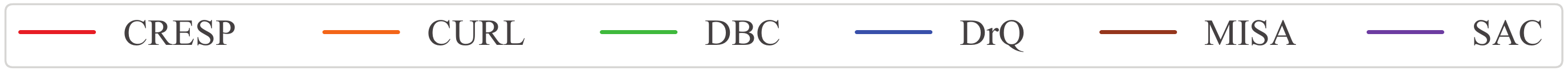}
		\vskip -0.05in
		\caption{Learning curves of six methods on six tasks with dynamic background distractions for 500K environment steps.
			The solid curves denote the means and the shaded regions denote the minimum and maximum returns over 6 trials.
			Each checkpoint is evaluated by 10 episodes on unseen environments.
			Curves are smoothed for visual clarity.}
		\Description{The performances of different algorithms on unseen environments with dynamic background distractions.}
		\label{fig-result-db}
	\end{figure*}

	\begin{algorithm}[tb]
		\caption{Characteristic Reward Sequence Prediction}
		\label{alg-cresp}
		\begin{algorithmic}
			\STATE {Initialize a replay buffer $\mathcal{D}$, a policy $\pi$, a representation $\Phi$, and a function approximator $\Psi$}
			\FOR {each iteration}
			\FOR {$e$ in $\mathcal{E}$}
			\FOR {each environment step $t$}
			\STATE Execute action $a_t\sim \pi(\cdot|\Phi(o_t))$
			\STATE Receive a transition $o_{t+1}, r_{t+1}\sim p^e(\cdot|o_t,a_t)$
			\STATE Record partial trajectories $\{(o_{t-i}, a_{t-i}, r_{t+1-i})\}_{i=0}^{T-1}$ in $\mathcal{D}$
			\ENDFOR
			\ENDFOR
			\FOR {each gradient step}
			\STATE Sample partial trajectories from $\mathcal{D}$
			\STATE Update the representation: \textcolor{blue}{$\mathcal{L^N_D}(\Phi, \Psi)$}
			\STATE Update the policy: $\mathcal{L}_{\text{RL}}(\pi)$
			\ENDFOR
			\ENDFOR
			\vskip -0.1in
		\end{algorithmic}
	\end{algorithm}

	\subsection{Characteristic Reward Sequence Prediction}
	\label{subsec-4.3}
	To improve the generalization of a learned policy on unseen environments with visual distractions, we propose  \textbf{C}haracteristic \textbf{Re}ward \textbf{S}equence \textbf{P}rediction (CRESP), a novel approach to learn representations for task relevance from high-dimensional observations.
	As discussed above, CRESP learns RSDs by predicting the characteristic functions $\varphi_{\mathbf{R}|o,\mathbf{a}}(\bm{\omega})$. In this section, we focus on the detailed learning procedure for the prediction.
	
	For an observation $o\in\mathcal{O}$ and an action sequence $\mathbf{a}\in\mathcal{A}^T$, the true characteristic function of the corresponding reward sequence is $\varphi_{\mathbf{R}|o,\mathbf{a}}(\bm{\omega})=\mathbb{E}_{\mathbf{R}\sim p(\cdot|o,\mathbf{a})}[e^{i\langle\bm{\omega},\mathbf{R}\rangle}]$.
	We estimate the characteristic function by a predictor $\Psi (\bm{\omega};\Phi(o),\mathbf{a})$.
	We use the weighted squared distance between the true and predicted characteristic functions as the prediction loss:
	\begin{align*}
		L^{\mathcal{W}}(\Phi, \Psi|o,\mathbf{a})\nonumber & = \mathbb{E}_{\bm{\Omega}\sim\mathcal{W}}\left[ \left\| \Psi\left(\bm{\Omega};\Phi(o), \mathbf{a}\right) - \varphi_{\mathbf{R}|o,\mathbf{a}}(\bm{\Omega}) \right\|^2_2 \right]\nonumber\\
		& = \int_{\mathbb{R}^T}\left|\Psi\left(\bm{\omega};\Phi(o), \mathbf{a}\right) - \varphi_{\mathbf{R}|o,\mathbf{a}}(\bm{\omega})\right|^2\mathcal{W}(\bm{\omega})\mathrm{d}\bm{\omega},
	\end{align*}
	where $\mathcal{W}$ is any probability density function on $\mathbb{R}^T$.
	We optimize the expected loss for observations and action sequences taken from the replay buffer $\mathcal{D}$:
	\begin{align*}
		L^{\mathcal{W}}_{\mathcal{D}}(\Phi,\Psi) & = \mathbb{E}_{(O,\mathbf{A})\sim\mathcal{D}}\left[L^{\mathcal{W}}(\Phi, \Psi|O,\mathbf{A})\right]\\
		& = \mathbb{E}_{(O,\mathbf{A})\sim\mathcal{D},\bm{\Omega}\sim\mathcal{W}}\left[ \left\| \Psi\left(\bm{\Omega};\Phi(O), \mathbf{A}\right) - \varphi_{\mathbf{R}|O,\mathbf{A}}(\bm{\Omega}) \right\|^2_2 \right].
	\end{align*}

	In practice, Since we have no access to the true characteristic functions, we propose to optimize an upper bound on $L^{\mathcal{W}}_{\mathcal{D}}$:
	\begin{align*}
		& \quad \mathcal{L}^{\mathcal{W}}_{\mathcal{D}}(\Phi, \Psi)\\
		& = \mathbb{E}_{(O, \mathbf{A}, \mathbf{R}) \sim \mathcal{D},\bm{\Omega}\sim\mathcal{W}}\left[ \left\| \Psi\left(\bm{\Omega}; \Phi(O), \mathbf{A}\right) - e^{i\langle\bm{\Omega},\mathbf{R}\rangle} \right\|^2_2 \right] \nonumber\\
		& \geq \mathbb{E}_{(O,\mathbf{A})\sim\mathcal{D},\bm{\Omega}\sim\mathcal{W}}\left[ \left\| \Psi\left(\bm{\Omega}; \Phi(O), \mathbf{A}\right) - \mathbb{E}_{\mathbf{R}\sim p(\cdot|O,\mathbf{A})}\left[e^{i\langle\bm{\Omega},\mathbf{R}\rangle}\right] \right\|^2_2 \right]\\
		& = L^{\mathcal{W}}_{\mathcal{D}}(\Phi,\Psi).
	\end{align*}
	Due to the complex form of characteristic functions, we divide the predictor $\Psi$ into two parts $\Psi=(\psi_{\cos}, \psi_{\sin})$, where $\psi_{\cos}$ estimates the real parts, and $\psi_{\sin}$ estimates the imaginary parts of characteristic functions, respectively.
	Moreover, we draw $\Omega$ from a Gaussian distribution $\mathcal{W}=\mathcal{N}(\bm{\mu},\bm{\sigma}^2)$ in practice. We then parameterize this distribution $\mathcal{N}(\bm{\mu},\bm{\sigma}^2)$ and perform ablation on it in Appendix~\ref{app-pred-loss}.
	Based on the  experimental results, we leverage the standard Gaussian distribution $\mathcal{N}$. Then the loss function is:
	\begin{align*}
		\mathcal{L}^\mathcal{N}_{\mathcal{D}}(\Phi,\Psi)
		=\mathbb{E}_{(O, \mathbf{A}, \mathbf{R}) \sim \mathcal{D},\bm{\Omega}\sim\mathcal{\mathcal{N}}} &\left[\left\| \psi_{\cos} \left(\bm{\Omega}; \Phi(O), \mathbf{A}\right) - \cos\left(\langle\bm{\Omega},\mathbf{R}\rangle\right) \right\|^2_2 \right. \nonumber\\
		+ &\left. \left\| \psi_{\sin} \left(\bm{\Omega}; \Phi(O), \mathbf{A}\right) - \sin\left(\langle\bm{\Omega},\mathbf{R}\rangle\right) \right\|^2_2 \right].
	\end{align*}
	
	In the training process, we update the encoder $\Phi$ and the predictor $\Psi$ due to the auxiliary loss $\mathcal{L}^{\mathcal{N}}_{\mathcal{D}}(\Phi,\Psi)$, and use the trained encoder $\Phi$ for the RL tasks.
	The whole architecture of CRESP and training procedure are illustrated in Figure~\ref{fig-arc} and Algorithm~\ref{alg-cresp}. 

	

	\section{Experiments}
	\label{sec-5}
	In this paper, we improve the performance of generalization on unseen environments with visual distractions.
	We focus on training agents in multi-environments under traditional off-policy settings without any prior environmental knowledge, such as strong augmentations designed for visual factors~\citep{lee2020network,iclr/ZhangCDL18,icml/FanWHY0ZA21}, fine-tuning in test environments~\citep{iclr/HansenJSAAEPW21}, or environmental labels for invariance~\citep{iclr/AgarwalMCB21,l4dc/SonarPM21}.
	We then investigate the performances of agents trained by different algorithms on various unseen test environments.
	
	For each environment, we benchmark CRESP extensively against prior state-of-the-art methods:
	1) \textbf{CURL}~\citep{icml/LaskinSA20}: a RL method with an auxiliary contrastive task;
	2) \textbf{DrQ}~\citep{iclr/YaratsKF21}: an effective method with state-of-the-art performance on DeepMind Control (DMControl)~\citep{corr/abs-1801-00690};
	3) \textbf{MISA}~\citep{icml/0001LSFKPGP20}: a recent approach from causal inference to learn invariant representations by approximating one-step rewards and dynamics;
	4) \textbf{DBC}~\citep{iclr/0001MCGL21}: a research for generalization in RL to learn representations via the bisimulation metric;
	5) \textbf{SAC}~\citep{icml/HaarnojaZAL18}: a traditional off-policy deep RL algorithm.
	
	\begin{table*}
		\caption{DMControl results with dynamic distractions at 500K steps. In dynamic background settings, all methods are evaluated on 30 unseen dynamic backgrounds. In dynamic color settings, $\beta_{\text{test}}=0.5$. Highest mean scores are boldfaced. CRESP outperforms prior SOTA methods in 11 out of 12 settings with +31.0\% boost on average.}
		\label{table-result-db}
		\centering
		\begin{tabular}{r|c|cccccc}
			\toprule
			& Method & Bic-catch & C-swingup & C-run & F-spin & R-easy & W-walk \\
			\midrule
			\multirow{6}{*}{\rotatebox{90}{Backgrounds}} & \textbf{CRESP} & $\textcolor{blue}{\mathbf{665} \pm \mathbf{185}}\,\,(\textcolor{blue}{+17\%})$ & $\textcolor{blue}{\mathbf{689} \pm \mathbf{49}}\,\,(\textcolor{blue}{+36\%})$ & $\textcolor{blue}{\mathbf{327} \pm \mathbf{54}}\,\,(\textcolor{blue}{+11\%})$ & $\textcolor{blue}{\mathbf{778} \pm \mathbf{154}}\,\,(\textcolor{blue}{+4\%})$ & $667 \pm 82$ & $\textcolor{blue}{\mathbf{794} \pm \mathbf{83}}\,\,(\textcolor{blue}{+6\%})$ \\
			& DrQ  & $570 \pm 126$ & $506 \pm 54$ & $295 \pm 24$ & $654 \pm 157$ & $683 \pm 177$ & $661 \pm 126$ \\
			& CURL & $167 \pm 142$ & $329 \pm 45$ & $185 \pm 39$ & $745 \pm 78$ & $\textcolor{blue}{\textbf{714} \pm \textbf{81}}$ & $746 \pm 41$ \\
			& DBC  & $113 \pm 133$ & $296 \pm 213$ & $133 \pm 98$ & $154 \pm 149$ & $129 \pm 64$ & $119 \pm 46$ \\
			& MISA & $123 \pm 44$ & $240 \pm 156$ & $178 \pm 21$ & $607 \pm 44$ & $360 \pm 91$ & $170 \pm 21$ \\
			& SAC  & $199 \pm 124$ & $209 \pm 14$ & $102 \pm 24$ & $188 \pm 114$ & $217 \pm 84$ & $96 \pm 62$ \\
			\midrule
			\multirow{6}{*}{\rotatebox{90}{Colors}} & \textbf{CRESP} & $\textcolor{blue}{\textbf{711} \pm \textbf{75}\,\,(+12\%)}$ & $\textcolor{blue}{\textbf{629} \pm \textbf{76}\,\,(+108\%)}$ & $\textcolor{blue}{\textbf{413} \pm \textbf{51}\,\,(+78\%)}$ & $\textcolor{blue}{\textbf{801} \pm \textbf{98}\,\,(+22\%)}$ & $\textcolor{blue}{\textbf{339} \pm \textbf{64}\,\,(+47\%)}$ & $\textcolor{blue}{\textbf{317} \pm \textbf{271}\,\,(+30\%)}$ \\
			& DrQ  & $634 \pm 85$ & $302 \pm 58$ & $216 \pm 144$ & $465 \pm 279$ & $160 \pm 52$ & $193 \pm 178$ \\
			& CURL & $470 \pm 128$ & $299 \pm 43$ & $232 \pm 35$ & $655 \pm 105$ & $231 \pm 82$ & $243 \pm 203$ \\
			& DBC  & $243 \pm 106$ & $129 \pm 18$ & $147 \pm 47$ & $25 \pm 34$ & $168 \pm 4$ & $140 \pm 57$ \\
			& MISA & $370 \pm 88$ & $169 \pm 1$ & $41 \pm 4$ & $2 \pm 1$ & $72 \pm 69$ & $94 \pm 3$ \\
			& SAC  & $244 \pm 136$ & $188 \pm 6$ & $87 \pm 45$ & $2 \pm 1$ & $109 \pm 46$ & $37 \pm 7$ \\
			\bottomrule
		\end{tabular}
	\end{table*}

	\paragraph{Network Details}
	Our method builds upon SAC and follows the network architecture of DrQ and CURL. 
	We use a 4-layer feed-forward ConvNet with no residual connection as the encoder. Then, we apply 3 fully connected layers with hidden size 1024 for actor and critic respectively.
	We also use the random cropping for image pre-processing proposed by DrQ as a weak augmentation without prior knowledge of test environments. To predict characteristic functions, we use 3 additional layers and ReLU after the pixel encoder. We present a detailed account of the architecture in Appendix~\ref{app-nd}.
	
	\paragraph{Experiment Parameters}
	In Section~\ref{subsec-exp-b} and Section~\ref{subsec-exp-c}, all experi-ments report the means and standard deviations of cumulative rewards over 6 trials per task for 500K environment steps.
	The experiments in Section~\ref{subsec-5.3} are on 3 random seeds.
	We apply the batch size 256 in CRESP and use the default values in the other methods respectively. 
	The action repeat of each task is adopted from Planet~\citep{icml/HafnerLFVHLD19}, which is the common setting in visual RL.
	Moreover, we choose the sequence length $T=5$ for better performance.
	For the computation of $\mathbb{E}_{\Omega\sim\mathcal{N}}\left[\cdot\right]$ in CRESP, we perform the ablation study of the sample number $\kappa$ of $\Omega$ in Appendix~\ref{app-pred-loss}, and we choose $\kappa=256$ for better performance.
	
	We evaluate the results of performance for 100 episodes in each table and boldface the results with the highest mean. In each figure, we plot the average cumulative rewards and the shaded region illustrates the standard deviation. Each checkpoint is evaluated using 10 episodes on unseen environments. More details of hyperparameters are in Appendix~\ref{app-pred-loss}.

	\subsection{Evaluation with Dynamic Backgrounds}
	\label{subsec-exp-b}
	To illustrate the effectiveness of our approach for generalization, we follow the benchmark settings in Distracting Control Suite (DCS)~\citep{corr/abs-2101-02722} with dynamic background distractions (Figure~\ref{fig-rl}) based on 6 DMControl tasks.
	We use 2 dynamic backgrounds during training and use the distracting background instead of original background ($\beta_{\text{bg}}=1.0$). We then evaluate generalization on 30 unseen dynamic backgrounds.
	We leverage the data from a replay buffer to approximate the reward sequence distributions in practice.
	Although we make the above approximation, CRESP performs well on the majority of tasks and advances MISA, the prior method similar with us by learning the properties of BMDP.
	On average of all tasks, CRESP achieves performance improvement by $12.3\%$ margins  (Table~\ref{table-result-db}) .
	Figure~\ref{fig-result-db} shows the learning curves of six methods.
	
	To visualize the representations learned by CRESP, we apply the t-distributed stochastic neighbor embedding (t-SNE) algorithm, a nonlinear dimensionality reduction technique to keep the similar high-dimensional vectors close in lower-dimensional space.
	Figure~\ref{fig-tsne} illustrates that in cheetah-run task, the representations learned by CRESP from different observations with similar states are the neighboring points in the two-dimensional map space.

	\begin{figure}
		\vskip -0.05in
		\includegraphics[width=8.5cm]{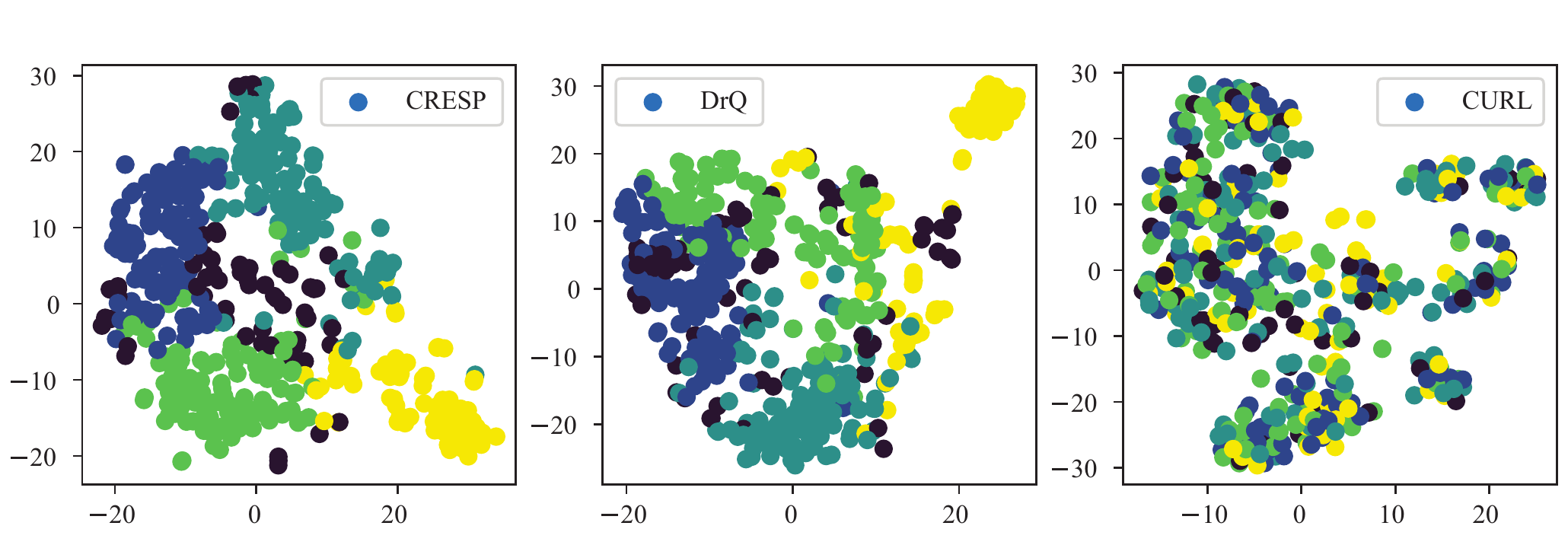}
		\vskip -0.1in
		\caption{t-SNE visualization of representations learned by CRESP (left), DrQ (center), and CURL (right). CRESP correctly groups semantically similar observations with different backgrounds, which are in the same colors.}
		\Description{t-SNE of latent spaces learned by CRESP and Drq.}
		\label{fig-tsne}
		\vskip -0.1in
	\end{figure}
	
	\subsection{Evaluation with Dynamic Color Distractions}
	\label{subsec-exp-c}
	To further demonstrate that CRESP can learn robust representations without task-irrelevant details in observations for generalization, we exhibit the performance on DCS with color distractions~\citep{corr/abs-2101-02722}.
	In these tasks, the color of the agent, which is the objective under control, changes during the episode. 
	The change of colors is modeled as a Gaussian distribution, whose mean is the color at the previous time and whose standard deviation is set as a hyperparameter $\beta$.
	We use 2 training environments with $\beta_1=0.1$ and $\beta_2=0.2$, and evaluate the agents in the test environment with $\beta_{\text{test}}=0.5$.
	
	We list the results in Table~\ref{table-result-db}. These results demonstrate that CRESP also improves the performance of generalization on unseen environment with color distractions and gains an average of $+49.7\%$ more return over the best prior method.

	\begin{figure}
		\centering
		\includegraphics[width=4.6cm]{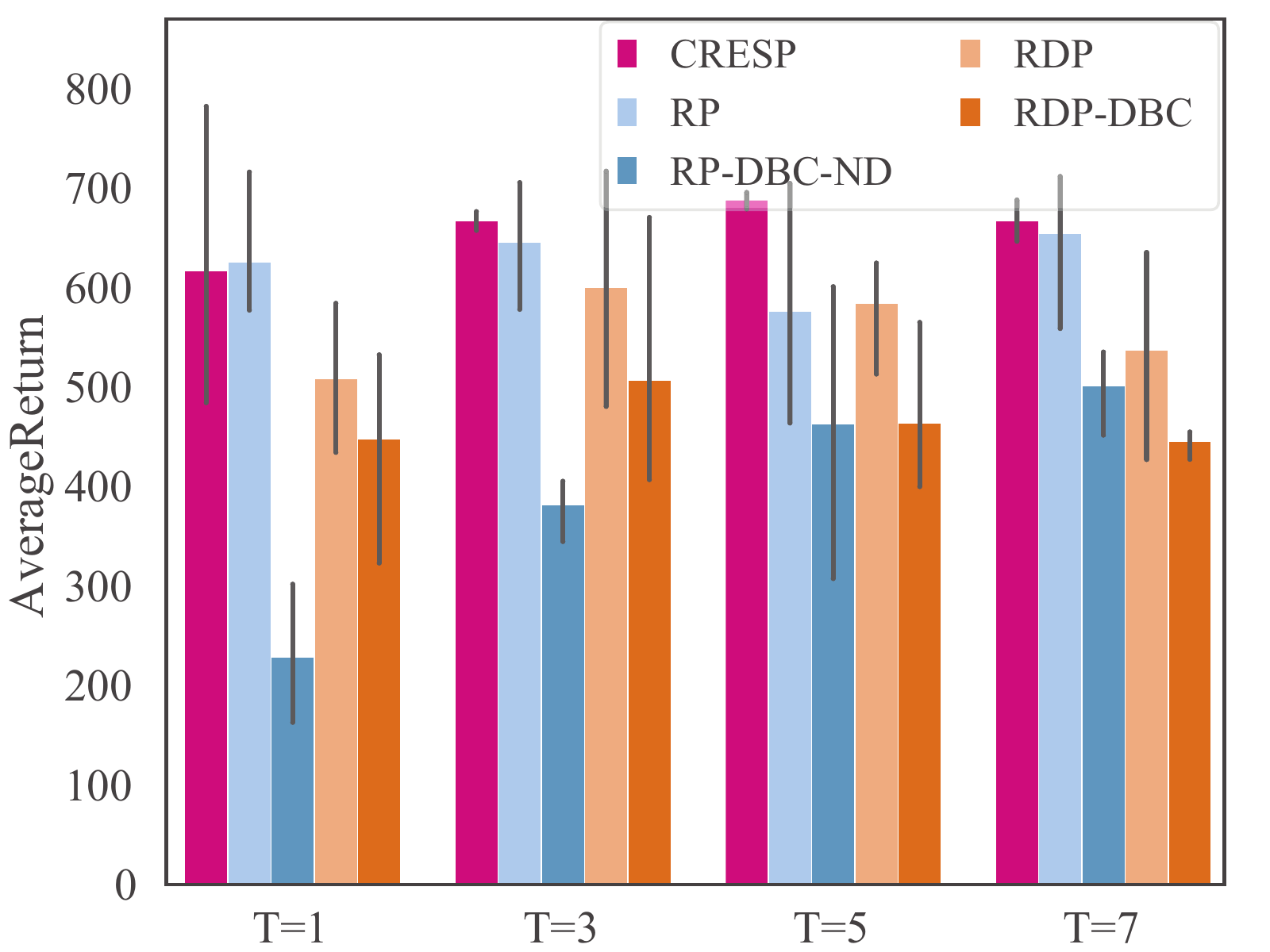}
		\vskip -0.1in
		\caption{Benchmarks in dynamic background settings on cartpole-swingup with 3 random seeds at 500K steps. Bars show means and standard errors.}
		\Description{The generalization performances of different methods leveraging rewards on unseen background dynamics.}
		\label{fig-rp-rdp-bism}
		\vskip -0.1in
	\end{figure}
	
	\subsection{Effectiveness of Reward Signals}
	\label{subsec-5.3}
	To understand the generalization performance of representations learned by using different properties of MDPs, such as reward signals and transition dynamics, we provide additional experiments on DCS with dynamic backgrounds to evaluate different representation learning methods:
	1) \textbf{CRESP}: predicting the characteristic functions of RSDs.
	2) \textbf{RP}: predicting the expectations of reward sequences. We apply a 3-layer MLP with $T$-dimensional outputs after the pixel encoder to estimate the $T$-dimensional expectations of reward sequences via $\ell_1$ distance.
	3) \textbf{RDP}: predicting the transition dynamics in addition to expectations of reward sequences.
	We apply a contrastive loss~\citep{corr/abs-1807-03748} to estimate the transition dynamics by maximizing the mutual information of representations from $o_t$ and $o_{t+1}$.
	4) \textbf{RDP-DBC}: based on RDP, also calculating bisimulation metrics~\citep{iclr/0001MCGL21}. We calculate bisimulation metrics in RDP and optimize the distances between representations to approximate their bisimulation metrics.
	5) \textbf{RP-DBC-ND}: based on RP, calculating bisimulation metrics, but without predicting transition dynamics.
	
	The discussion on performance of these methods on unseen envi-ronments is in Section~\ref{subsec-5.3.1}.
	To demonstrate the improvement of representations learned by CRESP, we visualize the task-irrelevant and task-relevant information in Section~\ref{subsec-5.3.3}.
	We present other results in Appendix~\ref{app-ar}.
	
	\subsubsection{Prediction of Rewards and Dynamics}
	\label{subsec-5.3.1}
	To compare the effect of reward signals and transition dynamics to generalization, we conduct experiments in $\textit{cheetah-run}$ and $\textit{cartpole-swingup}$ tasks with dynamic backgrounds to evaluate the agents learned by different methods: 1) RP; 2) RDP; 3) RP-DBC-ND; 4) RDP-DBC. 
	We also \mbox{ablate} the length $T$ of reward sequences from $1$ to $7$.
	Notice that RDP with $T=1$ is similar to MISA and RDP-DBC with $T=1$ is similar to DBC.
	In our experiments, all methods adopt the common architecture and hyperparamters other than the length $T$.
	According to the experimental results, the performance is best when $T=5$.
	
	Figure~\ref{fig-rp-rdp-bism} illustrates the performances on $\textit{cartpole-swingup}$, where CRESP outperforms others for long reward lengths (except for length $1$). 
	These results show that the auxiliary task for estimating RSDs in CRESP performs better than the task for estimating the expectations of reward sequences in RP.
	Moreover, the additional task for learning the transition dynamics may hinder the performance of generalization, because RP outperforms RDP for the same reward length except for length $5$.
	The performances of the methods using bisimulation metrics are lower than others, indicating that bisimulation metrics may be ineffective in RP and RDP.
	
	
	\subsubsection{The Empirical Evaluations of task relevance}
	\label{subsec-5.3.3}
	To demonstrate the generalization of representations, we design two experiments to quantify the task irrelevance and relevance of the the representations learned by: 1) DrQ; 2) CRESP; 3) RP; 4) RDP; 5) RDP-DBC.
	
	\paragraph{Problem Setup}
	We collect an offline dataset with 100K transitions (data of 800 episodes) drawn from 20 unseen environments.
	Then, we sample 80K transitions for training and the rest for evaluation.
	In each experiment, we apply a 3-layer MLP using ReLU activations up until the last layer. The inputs of the MLP are the learned representations at $500$K environment steps.
	We train the MLP by Adam with the learning rate 0.001, and we evaluate the results every 10 epochs.
	In Figure~\ref{fig-tr-tir}, we report the final values of cross-entropy of each algorithm with 3 random seeds.
	
	\begin{figure}
		\centering
		\vskip -0.05in
		\subfigure[Task irrelevance on C-swingup]{\includegraphics[width=3.4cm]{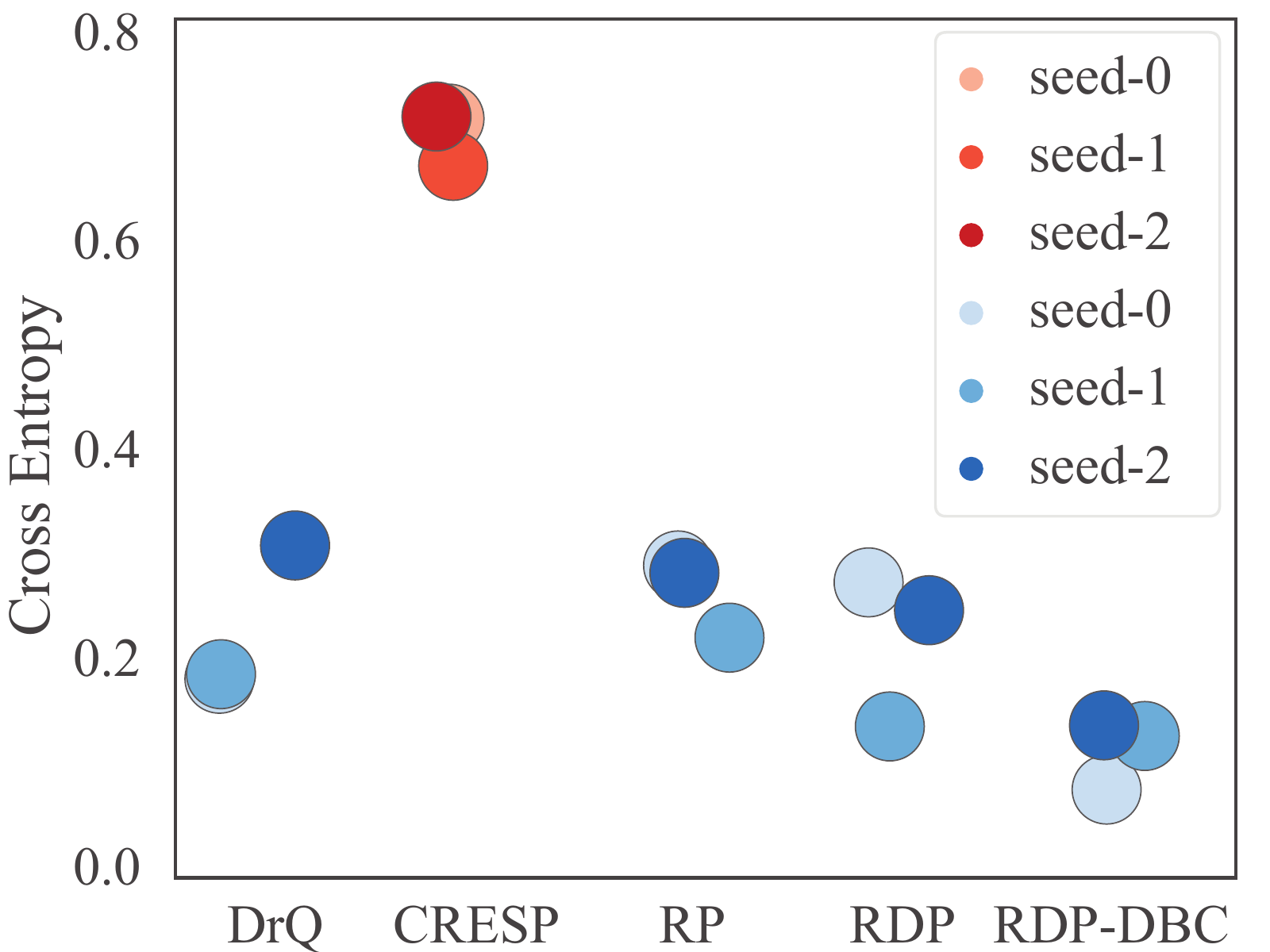}}
		\hskip 0.2in
		\subfigure[Task relevance on C-swingup]{\includegraphics[width=3.4cm]{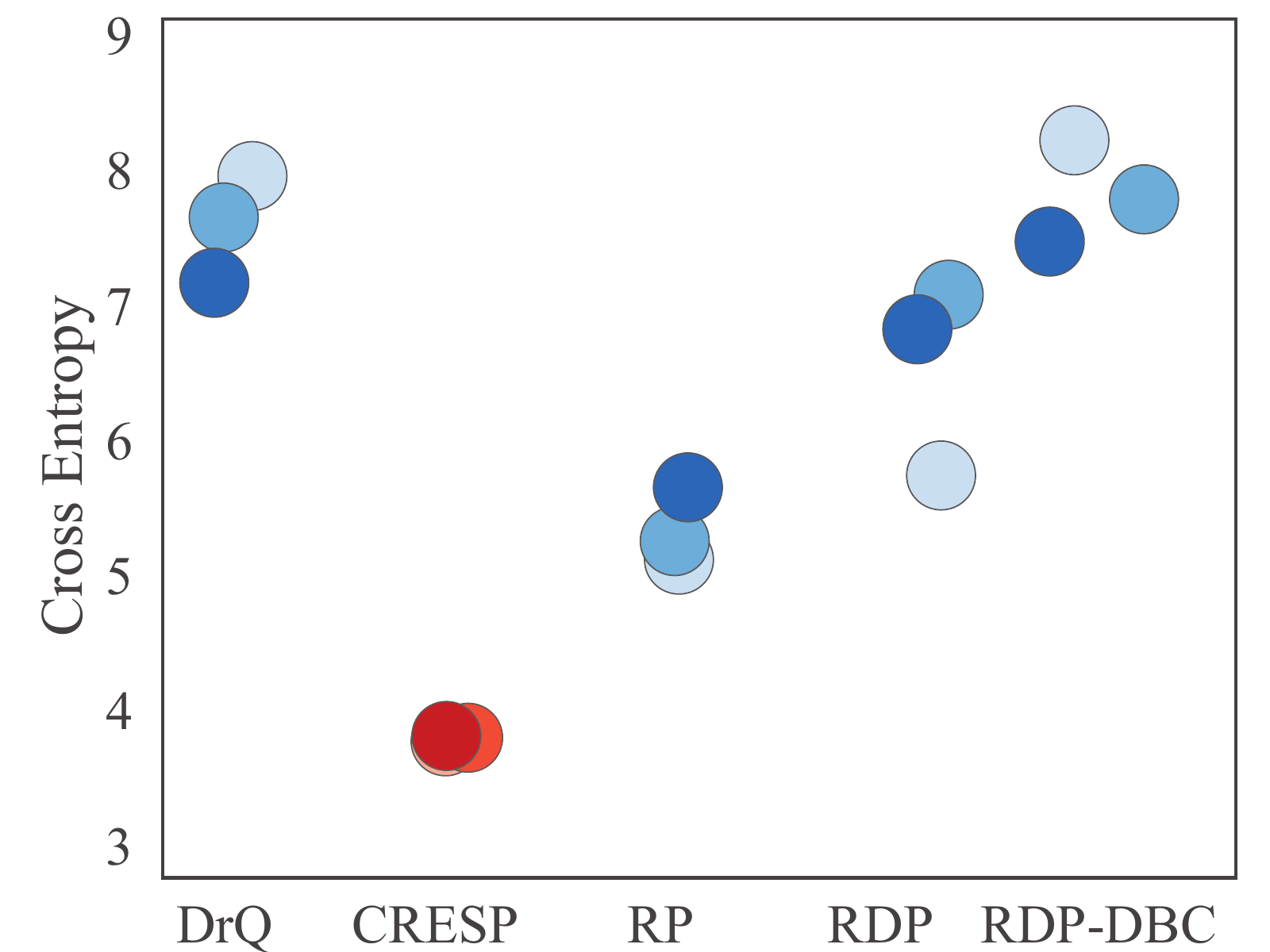}}
		
		\vskip -0.08in
		\subfigure[Task irrelevance on C-run]{\includegraphics[width=3.4cm]{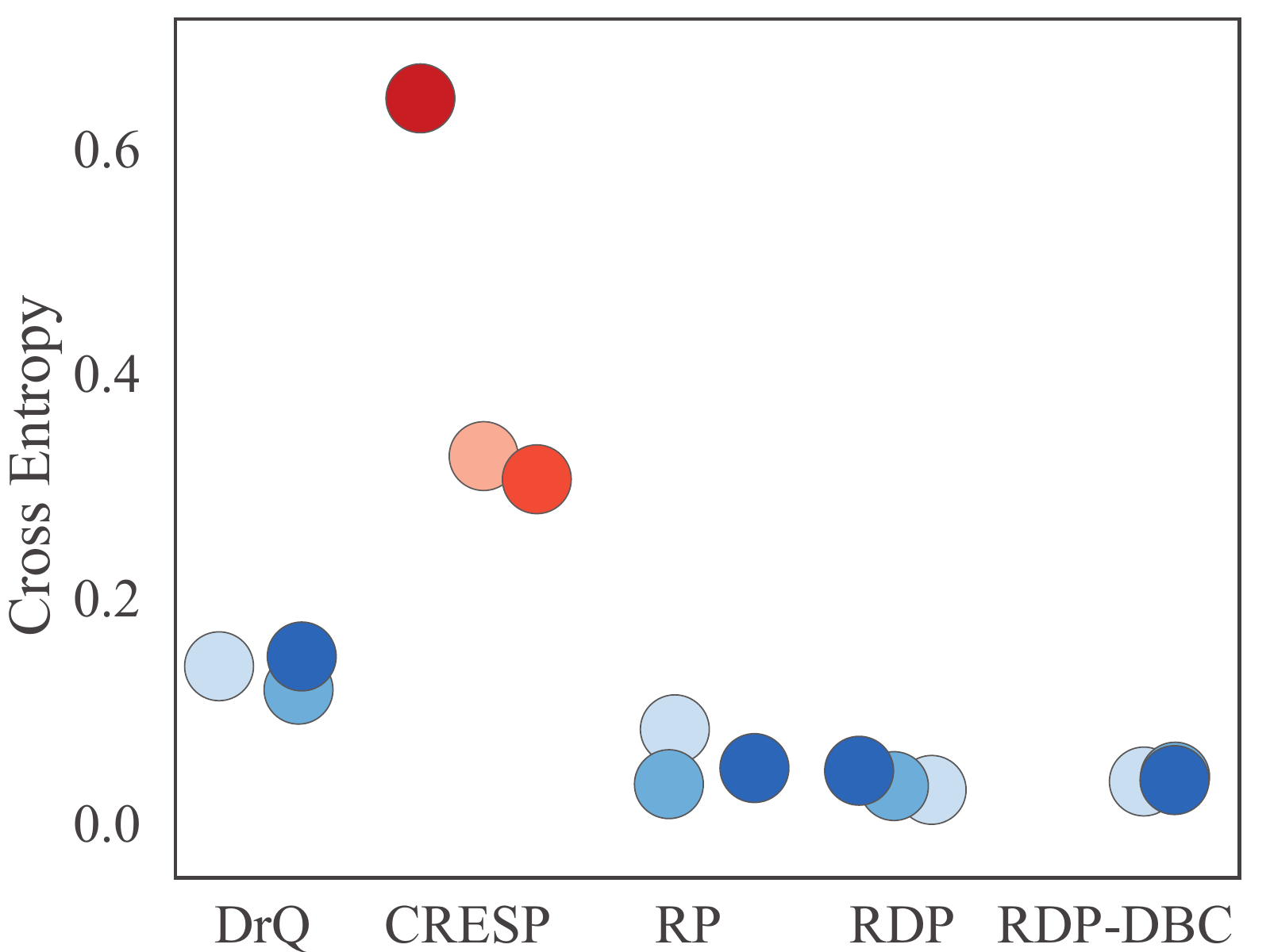}}
		\hskip 0.2in
		\subfigure[Task relevance on C-run]{\includegraphics[width=3.4cm]{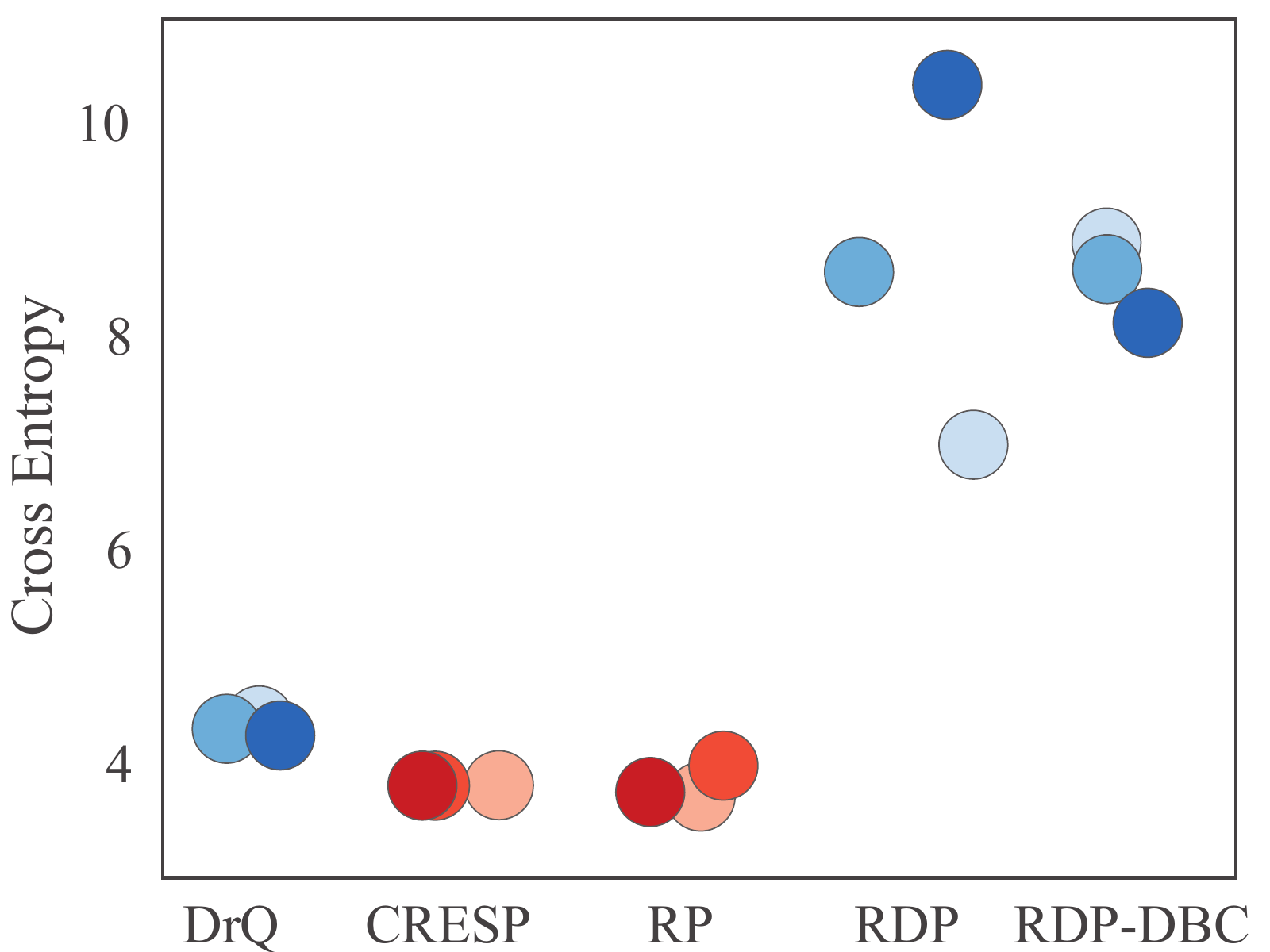}}
		\vskip -0.1in
		\caption{Results of generalization ability of representations. In the first column, low values indicate that the fixed representations are task-irrelevant. On the contrary, low values in the second column indicate that the representations are task-relevant. Best results in each part are shown in red.}
		\Description{Five methods are evaluated to quantify the task-relevant Information by cross-entropy.}
		\vskip -0.1in
		\label{fig-tr-tir}
	\end{figure}
	
	\paragraph{Task irrelevance of Representations}
	In the first column of Figure~\ref{fig-tr-tir}, we evaluate the task irrelevance of representations from different methods.
	We leverage the environmental label, an one-hot vector of the environment numbers from 1 to 20. 
	The information of the environmental label is explicitly task-irrelevant.
	Thus, we propose to measure the mutual information between the fixed representations and the random variable of the environmental label to quantify the task irrelevance.
	Based on contrastive learning, we update the MLP to estimate the mutual information by minimizing the cross-entropy between representations and environmental labels.
	A small cross-entropy indicates a high lower bound on the mutual information. 
	Therefore, a \textit{smaller} cross-entropy means that there is \textit{more task-irrelevant} information in the representations.
	
	The experimental results in the first column of Figure~\ref{fig-tr-tir} reveal that CRESP has less task irrelevance than other methods.
	However, the task-irrelevant information in the representations learned by RP is almost identical to that by RDP.
	This is the empirical evidence that the prediction task of characteristic functions of reward sequence distributions allows the representations to focus on task-relevant information and discard visual distractions. 

	\paragraph{Task relevance of Representations}
	We think of the state as a random vector from the simulation to provide the information of the agent.
	In principle, if the state information is presented in the pixel inputs, the visual RL algorithms should learn the representations to extract the information relevant to the state. Therefore, to evaluate the task relevance, we design another experiment by measuring the mutual information between the state and the learned representation from pixels.
	We also use the collected dataset and the technique to estimate the mutual information by minimizing the cross-entropy between learned representations and collected states.
	Different from the experiment of the task irrelevance, a \textit{smaller} cross-entropy loss means that there is more \textit{task-relevant} information in the learned representations.
	
	The second column in Figure~\ref{fig-tr-tir} shows that CRESP extracts the most task-relevant information by predicting characteristic functions than other methods.
	The representations learned by RP also encode more task-relevant information, but RDP and RDP-DBC extract less task-relevant information. This confirms our point in Section~\ref{subsec-4.1} that the reward sequence distributions can identify the task-relevant information, while predicting observation transition dynamics leads to an impact on this identification.
	

	

	\section{Conclusion}
	Generalization across different environments with the same tasks has been a great challenge in visual RL.
	To address this challenge, we propose a novel approach, namely CRESP, to learn task-relevant representations by leveraging the reward sequence distributions (RSDs).
	To effectively extract the task-relevant information for representation learning via RSDs, we develop an auxiliary task that predicts the characteristic functions of RSDs.
	Experiments on unseen environments with different visual distractions demonstrate the effectiveness of CRESP.
	We plan to extend the idea of CRESP to offline RL settings, which have broad applications in real scenarios.
	
	\section*{Acknowledgements}
	
	We would like to thank all the anonymous reviewers for their insightful comments. This work was supported in part by National Nature Science Foundations of China grants U19B2026, U19B2044, 61836011, and 61836006, and the Fundamental Research Funds for the Central Universities grant WK3490000004.

\bibliographystyle{ACM-Reference-Format}
\bibliography{sample-base}


\begin{thebibliography}{35}


\ifx \showCODEN    \undefined \def \showCODEN     #1{\unskip}     \fi
\ifx \showDOI      \undefined \def \showDOI       #1{#1}\fi
\ifx \showISBNx    \undefined \def \showISBNx     #1{\unskip}     \fi
\ifx \showISBNxiii \undefined \def \showISBNxiii  #1{\unskip}     \fi
\ifx \showISSN     \undefined \def \showISSN      #1{\unskip}     \fi
\ifx \showLCCN     \undefined \def \showLCCN      #1{\unskip}     \fi
\ifx \shownote     \undefined \def \shownote      #1{#1}          \fi
\ifx \showarticletitle \undefined \def \showarticletitle #1{#1}   \fi
\ifx \showURL      \undefined \def \showURL       {\relax}        \fi
\providecommand\bibfield[2]{#2}
\providecommand\bibinfo[2]{#2}
\providecommand\natexlab[1]{#1}
\providecommand\showeprint[2][]{arXiv:#2}

\bibitem[Agarwal et~al\mbox{.}(2021)]%
        {iclr/AgarwalMCB21}
\bibfield{author}{\bibinfo{person}{Rishabh Agarwal}, \bibinfo{person}{Marlos~C.
  Machado}, \bibinfo{person}{Pablo~Samuel Castro}, {and}
  \bibinfo{person}{Marc~G. Bellemare}.} \bibinfo{year}{2021}\natexlab{}.
\newblock \showarticletitle{Contrastive Behavioral Similarity Embeddings for
  Generalization in Reinforcement Learning}. In
  \bibinfo{booktitle}{\emph{{ICLR} 2021}}.
\newblock


\bibitem[Ansari et~al\mbox{.}(2020)]%
        {ansari2020characteristic}
\bibfield{author}{\bibinfo{person}{Abdul~Fatir Ansari},
  \bibinfo{person}{Jonathan Scarlett}, {and} \bibinfo{person}{Harold Soh}.}
  \bibinfo{year}{2020}\natexlab{}.
\newblock \showarticletitle{A Characteristic Function Approach to Deep Implicit
  Generative Modeling}. In \bibinfo{booktitle}{\emph{2020 {IEEE/CVF} Conference
  on Computer Vision and Pattern Recognition}}. \bibinfo{pages}{7476--7484}.
\newblock


\bibitem[Castro(2020)]%
        {aaai/Castro20}
\bibfield{author}{\bibinfo{person}{Pablo~Samuel Castro}.}
  \bibinfo{year}{2020}\natexlab{}.
\newblock \showarticletitle{Scalable Methods for Computing State Similarity in
  Deterministic Markov Decision Processes}. In \bibinfo{booktitle}{\emph{The
  Thirty-Fourth {AAAI} Conference on Artificial Intelligence}}.
  \bibinfo{pages}{10069--10076}.
\newblock


\bibitem[Du et~al\mbox{.}(2019)]%
        {du2019provably}
\bibfield{author}{\bibinfo{person}{Simon~S. Du}, \bibinfo{person}{Akshay
  Krishnamurthy}, \bibinfo{person}{Nan Jiang}, \bibinfo{person}{Alekh Agarwal},
  \bibinfo{person}{Miroslav Dud{\'{\i}}k}, {and} \bibinfo{person}{John
  Langford}.} \bibinfo{year}{2019}\natexlab{}.
\newblock \showarticletitle{Provably efficient {RL} with Rich Observations via
  Latent State Decoding}. In \bibinfo{booktitle}{\emph{{ICML} 2019}},
  Vol.~\bibinfo{volume}{97}. \bibinfo{pages}{1665--1674}.
\newblock


\bibitem[Espeholt et~al\mbox{.}(2018)]%
        {icml/EspeholtSMSMWDF18}
\bibfield{author}{\bibinfo{person}{Lasse Espeholt}, \bibinfo{person}{Hubert
  Soyer}, \bibinfo{person}{R{\'{e}}mi Munos}, \bibinfo{person}{Karen Simonyan},
  \bibinfo{person}{Volodymyr Mnih}, \bibinfo{person}{Tom Ward},
  \bibinfo{person}{Yotam Doron}, \bibinfo{person}{Vlad Firoiu},
  \bibinfo{person}{Tim Harley}, \bibinfo{person}{Iain Dunning},
  \bibinfo{person}{Shane Legg}, {and} \bibinfo{person}{Koray Kavukcuoglu}.}
  \bibinfo{year}{2018}\natexlab{}.
\newblock \showarticletitle{{IMPALA:} Scalable Distributed Deep-RL with
  Importance Weighted Actor-Learner Architectures}. In
  \bibinfo{booktitle}{\emph{{ICML} 2018}}. \bibinfo{pages}{1406--1415}.
\newblock


\bibitem[Fan et~al\mbox{.}(2021)]%
        {icml/FanWHY0ZA21}
\bibfield{author}{\bibinfo{person}{Linxi Fan}, \bibinfo{person}{Guanzhi Wang},
  \bibinfo{person}{De{-}An Huang}, \bibinfo{person}{Zhiding Yu},
  \bibinfo{person}{Li Fei{-}Fei}, \bibinfo{person}{Yuke Zhu}, {and}
  \bibinfo{person}{Animashree Anandkumar}.} \bibinfo{year}{2021}\natexlab{}.
\newblock \showarticletitle{{SECANT:} Self-Expert Cloning for Zero-Shot
  Generalization of Visual Policies}. In \bibinfo{booktitle}{\emph{{ICML}
  2021}}, Vol.~\bibinfo{volume}{139}. \bibinfo{pages}{3088--3099}.
\newblock


\bibitem[Farebrother et~al\mbox{.}(2018)]%
        {farebrother2018generalization}
\bibfield{author}{\bibinfo{person}{Jesse Farebrother},
  \bibinfo{person}{Marlos~C Machado}, {and} \bibinfo{person}{Michael Bowling}.}
  \bibinfo{year}{2018}\natexlab{}.
\newblock \showarticletitle{Generalization and regularization in DQN}.
\newblock \bibinfo{journal}{\emph{arXiv preprint arXiv:1810.00123}}
  (\bibinfo{year}{2018}).
\newblock


\bibitem[Grimmett and Stirzaker(2020)]%
        {grimmett2020probability}
\bibfield{author}{\bibinfo{person}{Geoffrey Grimmett} {and}
  \bibinfo{person}{David Stirzaker}.} \bibinfo{year}{2020}\natexlab{}.
\newblock \bibinfo{booktitle}{\emph{Probability and random processes}}.
\newblock \bibinfo{publisher}{Oxford university press}.
\newblock


\bibitem[Haarnoja et~al\mbox{.}(2018)]%
        {icml/HaarnojaZAL18}
\bibfield{author}{\bibinfo{person}{Tuomas Haarnoja}, \bibinfo{person}{Aurick
  Zhou}, \bibinfo{person}{Pieter Abbeel}, {and} \bibinfo{person}{Sergey
  Levine}.} \bibinfo{year}{2018}\natexlab{}.
\newblock \showarticletitle{Soft Actor-Critic: Off-Policy Maximum Entropy Deep
  Reinforcement Learning with a Stochastic Actor}. In
  \bibinfo{booktitle}{\emph{{ICML} 2018}}, Vol.~\bibinfo{volume}{80}.
  \bibinfo{pages}{1856--1865}.
\newblock


\bibitem[Hafner et~al\mbox{.}(2019)]%
        {icml/HafnerLFVHLD19}
\bibfield{author}{\bibinfo{person}{Danijar Hafner}, \bibinfo{person}{Timothy~P.
  Lillicrap}, \bibinfo{person}{Ian Fischer}, \bibinfo{person}{Ruben Villegas},
  \bibinfo{person}{David Ha}, \bibinfo{person}{Honglak Lee}, {and}
  \bibinfo{person}{James Davidson}.} \bibinfo{year}{2019}\natexlab{}.
\newblock \showarticletitle{Learning Latent Dynamics for Planning from Pixels}.
  In \bibinfo{booktitle}{\emph{{ICML} 2019}}, Vol.~\bibinfo{volume}{97}.
  \bibinfo{pages}{2555--2565}.
\newblock


\bibitem[Hansen et~al\mbox{.}(2021)]%
        {iclr/HansenJSAAEPW21}
\bibfield{author}{\bibinfo{person}{Nicklas Hansen}, \bibinfo{person}{Rishabh
  Jangir}, \bibinfo{person}{Yu Sun}, \bibinfo{person}{Guillem Aleny{\`{a}}},
  \bibinfo{person}{Pieter Abbeel}, \bibinfo{person}{Alexei~A. Efros},
  \bibinfo{person}{Lerrel Pinto}, {and} \bibinfo{person}{Xiaolong Wang}.}
  \bibinfo{year}{2021}\natexlab{}.
\newblock \showarticletitle{Self-Supervised Policy Adaptation during
  Deployment}. In \bibinfo{booktitle}{\emph{{ICLR} 2021}}.
\newblock


\bibitem[Igl et~al\mbox{.}(2019)]%
        {nips/IglCLTZDH19}
\bibfield{author}{\bibinfo{person}{Maximilian Igl}, \bibinfo{person}{Kamil
  Ciosek}, \bibinfo{person}{Yingzhen Li}, \bibinfo{person}{Sebastian
  Tschiatschek}, \bibinfo{person}{Cheng Zhang}, \bibinfo{person}{Sam Devlin},
  {and} \bibinfo{person}{Katja Hofmann}.} \bibinfo{year}{2019}\natexlab{}.
\newblock \showarticletitle{Generalization in Reinforcement Learning with
  Selective Noise Injection and Information Bottleneck}. In
  \bibinfo{booktitle}{\emph{NeurIPS 2019}}.
\newblock


\bibitem[Kalashnikov et~al\mbox{.}(2018)]%
        {corr/abs-1806-10293}
\bibfield{author}{\bibinfo{person}{Dmitry Kalashnikov}, \bibinfo{person}{Alex
  Irpan}, \bibinfo{person}{Peter Pastor}, \bibinfo{person}{Julian Ibarz},
  \bibinfo{person}{Alexander Herzog}, \bibinfo{person}{Eric Jang},
  \bibinfo{person}{Deirdre Quillen}, \bibinfo{person}{Ethan Holly},
  \bibinfo{person}{Mrinal Kalakrishnan}, \bibinfo{person}{Vincent Vanhoucke},
  {and} \bibinfo{person}{Sergey Levine}.} \bibinfo{year}{2018}\natexlab{}.
\newblock \showarticletitle{QT-Opt: Scalable Deep Reinforcement Learning for
  Vision-Based Robotic Manipulation}.
\newblock \bibinfo{journal}{\emph{CoRR}}  \bibinfo{volume}{abs/1806.10293}
  (\bibinfo{year}{2018}).
\newblock


\bibitem[Lange and Riedmiller(2010)]%
        {lange2010deep}
\bibfield{author}{\bibinfo{person}{Sascha Lange} {and}
  \bibinfo{person}{Martin~A. Riedmiller}.} \bibinfo{year}{2010}\natexlab{}.
\newblock \showarticletitle{Deep auto-encoder neural networks in reinforcement
  learning}. In \bibinfo{booktitle}{\emph{International Joint Conference on
  Neural Networks}}.
\newblock


\bibitem[Lange et~al\mbox{.}(2012)]%
        {lange2012autonomous}
\bibfield{author}{\bibinfo{person}{Sascha Lange}, \bibinfo{person}{Martin~A.
  Riedmiller}, {and} \bibinfo{person}{Arne Voigtl{\"{a}}nder}.}
  \bibinfo{year}{2012}\natexlab{}.
\newblock \showarticletitle{Autonomous reinforcement learning on raw visual
  input data in a real world application}. In \bibinfo{booktitle}{\emph{The
  2012 International Joint Conference on Neural Networks}}.
  \bibinfo{pages}{1--8}.
\newblock


\bibitem[Larsen and Skou(1991)]%
        {iandc/LarsenS91}
\bibfield{author}{\bibinfo{person}{Kim~Guldstrand Larsen} {and}
  \bibinfo{person}{Arne Skou}.} \bibinfo{year}{1991}\natexlab{}.
\newblock \showarticletitle{Bisimulation through Probabilistic Testing}.
\newblock \bibinfo{journal}{\emph{Inf. Comput.}} \bibinfo{volume}{94},
  \bibinfo{number}{1} (\bibinfo{year}{1991}), \bibinfo{pages}{1--28}.
\newblock


\bibitem[Laskin et~al\mbox{.}(2020a)]%
        {laskin2020reinforcement}
\bibfield{author}{\bibinfo{person}{Michael Laskin}, \bibinfo{person}{Kimin
  Lee}, \bibinfo{person}{Adam Stooke}, \bibinfo{person}{Lerrel Pinto},
  \bibinfo{person}{Pieter Abbeel}, {and} \bibinfo{person}{Aravind Srinivas}.}
  \bibinfo{year}{2020}\natexlab{a}.
\newblock \showarticletitle{Reinforcement Learning with Augmented Data}. In
  \bibinfo{booktitle}{\emph{NeurIPS 2020}}.
\newblock


\bibitem[Laskin et~al\mbox{.}(2020b)]%
        {icml/LaskinSA20}
\bibfield{author}{\bibinfo{person}{Michael Laskin}, \bibinfo{person}{Aravind
  Srinivas}, {and} \bibinfo{person}{Pieter Abbeel}.}
  \bibinfo{year}{2020}\natexlab{b}.
\newblock \showarticletitle{{CURL:} Contrastive Unsupervised Representations
  for Reinforcement Learning}. In \bibinfo{booktitle}{\emph{{ICML} 2020}}.
\newblock


\bibitem[Lee et~al\mbox{.}(2020)]%
        {lee2020network}
\bibfield{author}{\bibinfo{person}{Kimin Lee}, \bibinfo{person}{Kibok Lee},
  \bibinfo{person}{Jinwoo Shin}, {and} \bibinfo{person}{Honglak Lee}.}
  \bibinfo{year}{2020}\natexlab{}.
\newblock \showarticletitle{Network Randomization: {A} Simple Technique for
  Generalization in Deep Reinforcement Learning}. In
  \bibinfo{booktitle}{\emph{{ICLR} 2020}}.
\newblock


\bibitem[Lehnert et~al\mbox{.}(2020)]%
        {LehnertLF20}
\bibfield{author}{\bibinfo{person}{Lucas Lehnert}, \bibinfo{person}{Michael~L.
  Littman}, {and} \bibinfo{person}{Michael~J. Frank}.}
  \bibinfo{year}{2020}\natexlab{}.
\newblock \showarticletitle{Reward-predictive representations generalize across
  tasks in reinforcement learning}.
\newblock \bibinfo{journal}{\emph{PLoS Comput. Biol.}} \bibinfo{volume}{16},
  \bibinfo{number}{10} (\bibinfo{year}{2020}).
\newblock


\bibitem[Raileanu et~al\mbox{.}(2020)]%
        {raileanu2021automatic}
\bibfield{author}{\bibinfo{person}{Roberta Raileanu}, \bibinfo{person}{Max
  Goldstein}, \bibinfo{person}{Denis Yarats}, \bibinfo{person}{Ilya Kostrikov},
  {and} \bibinfo{person}{Rob Fergus}.} \bibinfo{year}{2020}\natexlab{}.
\newblock \showarticletitle{Automatic Data Augmentation for Generalization in
  Deep Reinforcement Learning}.
\newblock \bibinfo{journal}{\emph{CoRR}}  \bibinfo{volume}{abs/2006.12862}
  (\bibinfo{year}{2020}).
\newblock


\bibitem[Rajeswaran et~al\mbox{.}(2017)]%
        {nips/RajeswaranLTK17}
\bibfield{author}{\bibinfo{person}{Aravind Rajeswaran},
  \bibinfo{person}{Kendall Lowrey}, \bibinfo{person}{Emanuel Todorov}, {and}
  \bibinfo{person}{Sham~M. Kakade}.} \bibinfo{year}{2017}\natexlab{}.
\newblock \showarticletitle{Towards Generalization and Simplicity in Continuous
  Control}. In \bibinfo{booktitle}{\emph{{NeurIPS} 2017}}.
  \bibinfo{pages}{6550--6561}.
\newblock


\bibitem[Saengkyongam et~al\mbox{.}(2021)]%
        {abs-2106-00808}
\bibfield{author}{\bibinfo{person}{Sorawit Saengkyongam},
  \bibinfo{person}{Nikolaj Thams}, \bibinfo{person}{Jonas Peters}, {and}
  \bibinfo{person}{Niklas Pfister}.} \bibinfo{year}{2021}\natexlab{}.
\newblock \showarticletitle{Invariant Policy Learning: {A} Causal Perspective}.
\newblock \bibinfo{journal}{\emph{CoRR}}  \bibinfo{volume}{abs/2106.00808}
  (\bibinfo{year}{2021}).
\newblock


\bibitem[Sonar et~al\mbox{.}(2021)]%
        {l4dc/SonarPM21}
\bibfield{author}{\bibinfo{person}{Anoopkumar Sonar}, \bibinfo{person}{Vincent
  Pacelli}, {and} \bibinfo{person}{Anirudha Majumdar}.}
  \bibinfo{year}{2021}\natexlab{}.
\newblock \showarticletitle{Invariant Policy Optimization: Towards Stronger
  Generalization in Reinforcement Learning}. In
  \bibinfo{booktitle}{\emph{Proceedings of the 3rd Annual Conference on
  Learning for Dynamics and Control}}, Vol.~\bibinfo{volume}{144}.
  \bibinfo{pages}{21--33}.
\newblock


\bibitem[Song et~al\mbox{.}(2020)]%
        {iclr/SongJTDN20}
\bibfield{author}{\bibinfo{person}{Xingyou Song}, \bibinfo{person}{Yiding
  Jiang}, \bibinfo{person}{Stephen Tu}, \bibinfo{person}{Yilun Du}, {and}
  \bibinfo{person}{Behnam Neyshabur}.} \bibinfo{year}{2020}\natexlab{}.
\newblock \showarticletitle{Observational Overfitting in Reinforcement
  Learning}. In \bibinfo{booktitle}{\emph{{ICLR} 2020}}.
\newblock


\bibitem[Stone et~al\mbox{.}(2021)]%
        {corr/abs-2101-02722}
\bibfield{author}{\bibinfo{person}{Austin Stone}, \bibinfo{person}{Oscar
  Ramirez}, \bibinfo{person}{Kurt Konolige}, {and} \bibinfo{person}{Rico
  Jonschkowski}.} \bibinfo{year}{2021}\natexlab{}.
\newblock \showarticletitle{The Distracting Control Suite - {A} Challenging
  Benchmark for Reinforcement Learning from Pixels}.
\newblock \bibinfo{journal}{\emph{CoRR}}  \bibinfo{volume}{abs/2101.02722}
  (\bibinfo{year}{2021}).
\newblock


\bibitem[Tassa et~al\mbox{.}(2018)]%
        {corr/abs-1801-00690}
\bibfield{author}{\bibinfo{person}{Yuval Tassa}, \bibinfo{person}{Yotam Doron},
  \bibinfo{person}{Alistair Muldal}, \bibinfo{person}{Tom Erez},
  \bibinfo{person}{Yazhe Li}, \bibinfo{person}{Diego de Las~Casas},
  \bibinfo{person}{David Budden}, \bibinfo{person}{Abbas Abdolmaleki},
  \bibinfo{person}{Josh Merel}, \bibinfo{person}{Andrew Lefrancq},
  \bibinfo{person}{Timothy~P. Lillicrap}, {and} \bibinfo{person}{Martin~A.
  Riedmiller}.} \bibinfo{year}{2018}\natexlab{}.
\newblock \showarticletitle{DeepMind Control Suite}.
\newblock \bibinfo{journal}{\emph{CoRR}}  \bibinfo{volume}{abs/1801.00690}
  (\bibinfo{year}{2018}).
\newblock


\bibitem[van~den Oord et~al\mbox{.}(2018)]%
        {corr/abs-1807-03748}
\bibfield{author}{\bibinfo{person}{A{\"{a}}ron van~den Oord},
  \bibinfo{person}{Yazhe Li}, {and} \bibinfo{person}{Oriol Vinyals}.}
  \bibinfo{year}{2018}\natexlab{}.
\newblock \showarticletitle{Representation Learning with Contrastive Predictive
  Coding}.
\newblock \bibinfo{journal}{\emph{CoRR}}  \bibinfo{volume}{abs/1807.03748}
  (\bibinfo{year}{2018}).
\newblock


\bibitem[Wang et~al\mbox{.}(2021)]%
        {corr/abs-2112-10504}
\bibfield{author}{\bibinfo{person}{Zhihai Wang}, \bibinfo{person}{Jie Wang},
  \bibinfo{person}{Qi Zhou}, \bibinfo{person}{Bin Li}, {and}
  \bibinfo{person}{Houqiang Li}.} \bibinfo{year}{2021}\natexlab{}.
\newblock \showarticletitle{Sample-Efficient Reinforcement Learning via
  Conservative Model-Based Actor-Critic}.
\newblock \bibinfo{journal}{\emph{CoRR}}  \bibinfo{volume}{abs/2112.10504}
  (\bibinfo{year}{2021}).
\newblock


\bibitem[Yarats et~al\mbox{.}(2021)]%
        {iclr/YaratsKF21}
\bibfield{author}{\bibinfo{person}{Denis Yarats}, \bibinfo{person}{Ilya
  Kostrikov}, {and} \bibinfo{person}{Rob Fergus}.}
  \bibinfo{year}{2021}\natexlab{}.
\newblock \showarticletitle{Image Augmentation Is All You Need: Regularizing
  Deep Reinforcement Learning from Pixels}. In \bibinfo{booktitle}{\emph{{ICLR}
  2021}}.
\newblock


\bibitem[Ye et~al\mbox{.}(2020)]%
        {ye2020rotation}
\bibfield{author}{\bibinfo{person}{Chang Ye}, \bibinfo{person}{Ahmed Khalifa},
  \bibinfo{person}{Philip Bontrager}, {and} \bibinfo{person}{Julian Togelius}.}
  \bibinfo{year}{2020}\natexlab{}.
\newblock \showarticletitle{Rotation, Translation, and Cropping for Zero-Shot
  Generalization}. In \bibinfo{booktitle}{\emph{{IEEE} Conference on Games}}.
  \bibinfo{pages}{57--64}.
\newblock


\bibitem[Yu(2004)]%
        {yu2004empirical}
\bibfield{author}{\bibinfo{person}{Jun Yu}.} \bibinfo{year}{2004}\natexlab{}.
\newblock \showarticletitle{Empirical characteristic function estimation and
  its applications}.
\newblock \bibinfo{journal}{\emph{Econometric reviews}} \bibinfo{volume}{23},
  \bibinfo{number}{2} (\bibinfo{year}{2004}), \bibinfo{pages}{93--123}.
\newblock


\bibitem[Zhang et~al\mbox{.}(2020)]%
        {icml/0001LSFKPGP20}
\bibfield{author}{\bibinfo{person}{Amy Zhang}, \bibinfo{person}{Clare Lyle},
  \bibinfo{person}{Shagun Sodhani}, \bibinfo{person}{Angelos Filos},
  \bibinfo{person}{Marta Kwiatkowska}, \bibinfo{person}{Joelle Pineau},
  \bibinfo{person}{Yarin Gal}, {and} \bibinfo{person}{Doina Precup}.}
  \bibinfo{year}{2020}\natexlab{}.
\newblock \showarticletitle{Invariant Causal Prediction for Block MDPs}. In
  \bibinfo{booktitle}{\emph{{ICML} 2020}}, Vol.~\bibinfo{volume}{119}.
  \bibinfo{pages}{11214--11224}.
\newblock


\bibitem[Zhang et~al\mbox{.}(2021)]%
        {iclr/0001MCGL21}
\bibfield{author}{\bibinfo{person}{Amy Zhang}, \bibinfo{person}{Rowan~Thomas
  McAllister}, \bibinfo{person}{Roberto Calandra}, \bibinfo{person}{Yarin Gal},
  {and} \bibinfo{person}{Sergey Levine}.} \bibinfo{year}{2021}\natexlab{}.
\newblock \showarticletitle{Learning Invariant Representations for
  Reinforcement Learning without Reconstruction}. In
  \bibinfo{booktitle}{\emph{{ICLR} 20211}}.
\newblock


\bibitem[Zhang et~al\mbox{.}(2018)]%
        {iclr/ZhangCDL18}
\bibfield{author}{\bibinfo{person}{Hongyi Zhang}, \bibinfo{person}{Moustapha
  Ciss{\'{e}}}, \bibinfo{person}{Yann~N. Dauphin}, {and} \bibinfo{person}{David
  Lopez{-}Paz}.} \bibinfo{year}{2018}\natexlab{}.
\newblock \showarticletitle{mixup: Beyond Empirical Risk Minimization}. In
  \bibinfo{booktitle}{\emph{{ICLR} 2018}}.
\newblock


\end{thebibliography}

\appendix

\section{Proofs}
\begin{theorem}
	\label{app-thm:bd}
	Let $\Phi:\mathcal{O}\to\mathcal{Z}$ be a $T$-level representation, $V^e_*:\mathcal{O}\to\mathbb{R}$ be the optimal value function in the environment $e\in\mathcal{E}$, and $\Bar{V}^e_*:\mathcal{Z}\to\mathbb{R}$ be the optimal value function on the latent representation space, built on top of the representation $\Phi$. Let $\bar{r}$ be a bound of the reward space, i.e., $|r|<\bar{r}$ for any $r\in\mathcal{R}$.
	Then we have
	\begin{align*}
		0\le V^e_*(o)-\Bar{V}^e_*\circ\Phi(o)\le\frac{2\gamma^T}{1-\gamma}\bar{r}
	\end{align*}
	for any $o\in\mathcal{O}$ and $e\in\mathcal{E}$.
\end{theorem}
\begin{proof}
	By the definition of optimal value function, obviously we have $\Bar{V}^e_*\circ\Phi(o)\le V^e_*(o)$.
	It suffices to show that $V^e_*(o)-\Bar{V}^e_*\circ\Phi(o)\le\frac{\gamma^T}{1-\gamma}\bar{r}$.
	Without loss of generality, let $\pi_*$ be an optimal policy such that $\pi_*(\cdot|o)=\pi_*(\cdot|o')$ if $[o]_s=[o']_s$. 
	Let $\hat{\pi}$ be any policy built on top of $\Phi$.
	Then we have
	\begin{align*}
		&V^e_*(o)-\Bar{V}^e_*\circ\Phi(o)\le \mathbb{E}_{\pi_*}^e\left[\sum_{t=0}^\infty\gamma^tR_{t+1}\right]-\mathbb{E}_{\hat{\pi}\circ\Phi}^e\left[\sum_{t=0}^\infty\gamma^tR_{t+1}\right]\\
		\le& \mathbb{E}_{\pi_*}^e\left[\sum_{t=0}^{T-1}\gamma^tR_{t+1}\right] - \mathbb{E}_{\hat{\pi}\circ\Phi}^e\left[\sum_{t=0}^{T-1}\gamma^tR_{t+1}\right]\\
		&+ \sum_{t=T}^\infty \gamma^t \left|\mathbb{E}_{\pi_*}^e\left[R_{t+1}\right]-\mathbb{E}_{\hat{\pi}\circ\Phi}^e\left[R_{t+1}\right]\right|\\
		\le& \mathbb{E}_{\pi_*}^e\left[\sum_{t=0}^{T-1}\gamma^tR_{t+1}\right] - \mathbb{E}_{\hat{\pi}\circ\Phi}^e\left[\sum_{t=0}^{T-1}\gamma^tR_{t+1}\right]+ 2\sum_{t=T}^\infty \gamma^t \bar{r}\\
		\le& \mathbb{E}_{\pi_*}^e\left[\sum_{t=0}^{T-1}\gamma^tR_{t+1}\right] - \mathbb{E}_{\hat{\pi}\circ\Phi}^e\left[\sum_{t=0}^{T-1}\gamma^tR_{t+1}\right] +\frac{2\gamma^T}{1-\gamma}\bar{r}.
	\end{align*}
	We are now to show that there exists a $\hat{\pi}$ such that \begin{align*}
		\mathbb{E}_{\pi_*}^e\left[\sum_{t=0}^{T-1}\gamma^tR_{t+1}\right] = \mathbb{E}_{\hat{\pi}\circ\Phi}^e\left[\sum_{t=0}^{T-1}\gamma^tR_{t+1}\right].
	\end{align*}
	For any two observations $o,o'\in\mathcal{O}$ such that $\Phi(o)=\Phi(o')$, since $\Phi$ is a $T$-level reward sequence representation, the RSDs $p(\mathbf{r}|o,\mathbf{a})=p(\mathbf{r}|o',\mathbf{a})$.
	Thus we have 
	\begin{align*}
		\mathbb{E}_{\hat{\pi}\circ\Phi}^e\left[\sum_{t=0}^{T-1}\gamma^tR_{t+1}|O_0=o\right]=\mathbb{E}_{\hat{\pi}\circ\Phi}^e\left[\sum_{t=0}^{T-1}\gamma^tR_{t+1}|O_0=o'\right]
	\end{align*} for any $\hat{\pi}$.
	Define $\hat{\pi}:\mathcal{Z}\to\Delta(\mathcal{R}^T)$ as $\hat{\pi}(z)=\pi_*(\bar{o})$, where $\bar{o}$ is an representative observation such that $\Phi(\bar{o})=z$. 
	Then we have
	\begin{align*}
		\mathbb{E}_{\hat{\pi}\circ\Phi}^e\left[\sum_{t=0}^{T-1}\gamma^tR_{t+1}|O_0=o\right]=&\mathbb{E}_{\hat{\pi}\circ\Phi}^e\left[\sum_{t=0}^{T-1}\gamma^tR_{t+1}|O_0=\bar{o}\right]\\
		=&\mathbb{E}_{\pi_*}^e\left[\sum_{t=0}^{T-1}\gamma^tR_{t+1}\right],
	\end{align*}
	which completes the proof.
\end{proof}

\begin{theorem}
	\label{app-thm:cf}
	A representation $\Phi:\mathcal{O}\to\mathcal{Z}$ is a $T$-level reward sequence representation if and only if there exits a predictor $\Psi$ such that
	\begin{align*}
		\Psi(\bm{\omega};\Phi(o),\mathbf{a})=\varphi_{\mathbf{R}|o,\mathbf{a}}(\bm{\omega})=\mathbb{E}_{\mathbf{R}\sim p(\cdot|o,\mathbf{a})}\left[e^{i\langle\bm{\omega},\mathbf{R}\rangle}\right],
	\end{align*}
	for all $\bm{w}\in\mathbb{R}^{T},o\in\mathcal{O}$ and $\mathbf{a}\in\mathcal{A}^T$.
\end{theorem}
\begin{proof}
	By definition, $\Phi$ is a $T$-level reward sequence representation if and only if there exists $f$ such that
	\begin{align*}
		f(\mathbf{r};\Phi(o),\mathbf{a})=p(\mathbf{r}|o,\mathbf{a}).
	\end{align*}
	Let $\mathcal{M}$ be the set of mappings from $\mathcal{A}^T$ to $\Delta(\mathcal{R}^T)$. Then $\Phi$ is a $T$-level reward sequence representation if and only if there exists $\hat{f}:\mathcal{Z}\to\mathcal{M}$ such that $\hat{f}\circ\Phi (o)=M_o$ for any $o\in\mathcal{O}$, where $M_o(\mathbf{a})=p(\cdot|o,\mathbf{a})$ for any $\mathbf{a}\in\mathcal{A}^T$.
	Since a characteristic function uniquely determines a distribution, and vice versa, there exists a bijection between the distributions $p(\cdot|o,\mathbf{a})$ and their corresponding characteristic functions $\varphi_{\mathbf{R}|o,\mathbf{a}}(\bm{\omega})=\mathbb{E}_{\mathbf{R}\sim p(\cdot|o,\mathbf{a})}\left[e^{i\langle\bm{\omega},\mathbf{R}\rangle}\right]$.
	As a result, $\Phi$ is a $T$-level reward sequence representation if and only if there exists $\Tilde{f}:\mathcal{Z}\to\mathcal{M}$ such that $\Tilde{f}\circ\Phi (o)=\Tilde{M}_o$ for any $o\in\mathcal{O}$, where $M_o(\mathbf{a})=\varphi_{\mathbf{R}|o,\mathbf{a}}(\cdot)$ for any $\mathbf{a}\in\mathcal{A}^T$, which is equivalent to what we want.
\end{proof}

\section{Research Methods}
\label{app-rm}

\subsection{Implementation Details}

\paragraph{Dynamic Background Distractions}
We follow the dynamic background settings of Distracting Control Suite (DCS)~\citep{corr/abs-2101-02722}.
We take the 2 first videos from the DAVIS 2017 training set and randomly sample a scene and a frame from those at the start of every episode. In the RL training process, we alternate the 2 videos. Moreover, we set $\beta_{\text{bg}}=1.0$, which means we use the distracting background instead of original skybox.
We then apply the 30 videos from the DAVIS 2017 validation dataset for evaluation.

\paragraph{Dynamic Color Distractions}
In dynamic color settings of DCS~\citep{corr/abs-2101-02722}, the color is sampled uniformly per channel $x_{0} \sim \mathcal{U}(x-\beta, x+\beta)$ at the start of each episode, where $x$ is the original color in DCS, and $\beta$ is a hyperparameter. We enable the dynamic setting that the color $x_t$ changes to $x_{t+1} = \text{clip} (\hat{x}_{t+1}, x_t - \beta, x_t + \beta)$, where $\hat{x}_{t+1} \sim \mathcal{N}(x_t, 0.03\cdot\beta)$. During the training process, we use $\beta_1=0.1$, $\beta_2=0.2$ and evaluate agents with $\beta_{\text{test}}=0.5$.

\paragraph{Network Details}
\label{app-nd}
A shared pixel encoder utilizes four convolutional layers using 3 × 3 kernels and 32 filters with an stride of 2 for the first layer and 1 for others. Rectified linear unit (ReLU) activations are applied after each convolution. Thereafter, a 50 dimensional output dense layer normalized by LayerNorm is applied with a $\tanh$ activation. Both critic and actor networks are parametrized with 3 fully connected layers using ReLU activations up until the last layer. The output dimension of these hidden layers is 1024.  The pixel encoder weights are shared for the critic and the actor, and gradients of the encoder are not computed through the actor optimizer.
Moreover, we use the random cropping for image pre-processing proposed by DrQ and RAD~\citep{laskin2020reinforcement} as a weak augmentation without prior knowledge of test environments.
For the predictor of characteristic functions, we use 3 additional layers and ReLU after the pixel encoder.

\begin{figure}[h]
	\centering
	\hspace{-0.1in}
	\includegraphics[width=2.4cm]{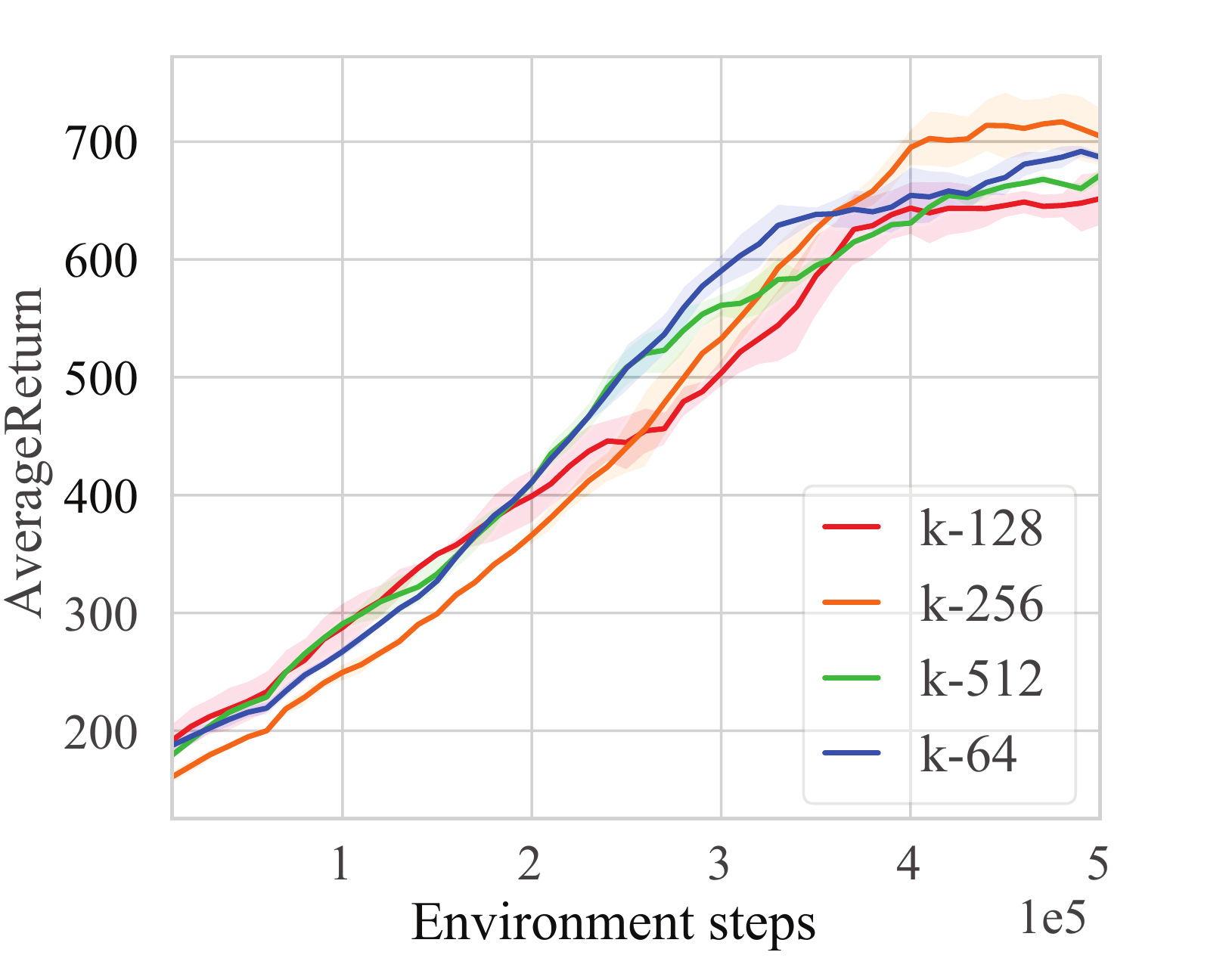}
	\hskip 0.1in
	\includegraphics[width=2.4cm]{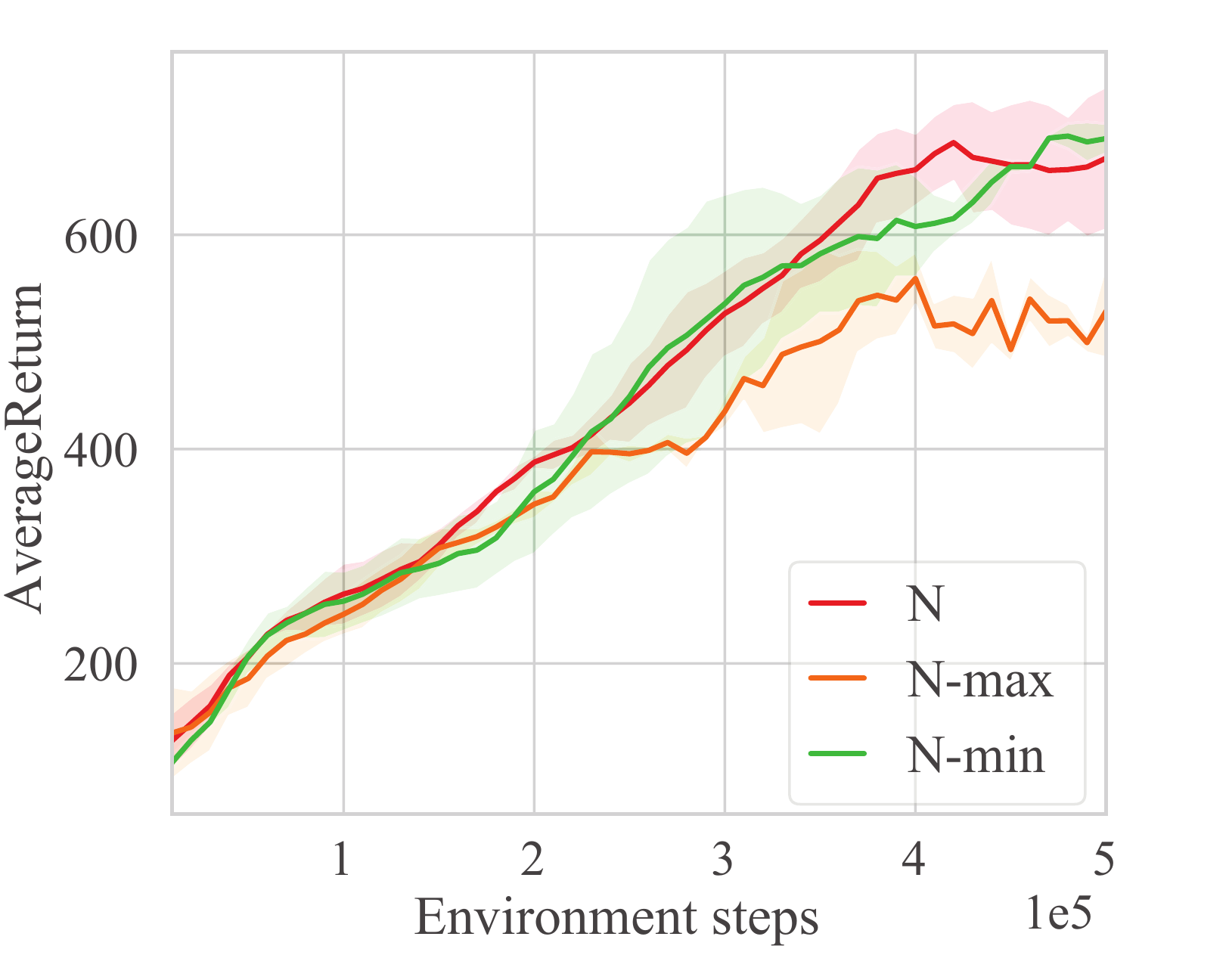}
	\vskip -0.1in
	\caption{Ablation studies of CRESP with batch size 128 and 2 random seeds on cartpole-swingup with dynamic backgrounds. In the left, $\kappa$ is sample number of $\Omega$ from $\mathcal{N}$. In the right, N is CRESP with the standard Gaussian distribution. N-max and N-min are CRESP with Gaussian distributions that maximize and minimize the auxiliary loss, respectively.}
	\Description{The generalization performance of CRESP on unseen background dynamics.}
	\label{fig-k}
\end{figure}

\paragraph{Auxiliary Loss in CRESP}
\label{app-pred-loss}
For all our experiments, we estimate the auxiliary loss $\mathcal{L}^{\mathcal{N}}_{\mathcal{D}}(\Phi,\Psi)$ to learn representations. For the computation of $\mathbb{E}_{\Omega\sim\mathcal{N}}\left[\cdot\right]$ in $\mathcal{L}^{\mathcal{N}}_{\mathcal{D}}(\Phi,\Psi)$, we take an ablation study of the sample number $\kappa$ of $\Omega$ in the left of Figure~\ref{fig-k}. The results show that $\kappa=256$ performs better than others. Then, for the choice of the distribution $\mathcal{N}$, we also parameterize this distribution to maximize the auxiliary loss (N-max) or minimize (N-min). The results with batch size 128 in the right of Figure~\ref{fig-k} indicate that PN-max is not stable and PN-min is similar to CRESP with the standard Gaussian distribution. For computational efficiency, we choose the standard Gaussian distribution in all experiments.

Furthermore, we list the hyperparameters in Table~\ref{table-hypp}. 

\begin{table}[h]
	\caption{Hyperparameters in the Distracting Control Suite.}
	\vskip -0.1in
	\label{table-hypp}
	\centering
	\begin{tabular}{l|l}
		\toprule
		\textbf{Hyperparameter}             & \textbf{Setting} \\
		\midrule
		Optimizer                           & Adam \\
		Discount $\gamma$                   & 0.99 \\
		Learning rate                       & 0.0005 \\
		Number of batch size                & 256 \\
		Number of hidden layers             & 2 \\
		Number of hidden units per layer    & 1024 \\
		Replay buffer size                  & 100,000 \\
		Initial steps                       & 1000 \\
		Target smoothing coefficient $\tau$ & 0.01 \\
		Critic target update frequency      & 2 \\
		Actor update frequency              & 1 \\
		Actor log stddev bounds             & [-5, 2] \\
		Initial temperature $\alpha$        & 0.1 \\
		\midrule
		\emph{Hyperparameters of CRESP}     & \\
		\quad Gaussian distribution $\mathcal{N}$ & $\mathcal{N}(\textbf{0}, \textbf{1})$\\
		\quad Sample number $\kappa$ from $\mathcal{N}$ & 256 \\
		\quad Discount of reward sequences & 0.8 \\
		\bottomrule
	\end{tabular}
\end{table}

\subsection{Additional Results}
\label{app-ar}

\begin{figure}
	\centering
	\includegraphics[width=8.6cm]{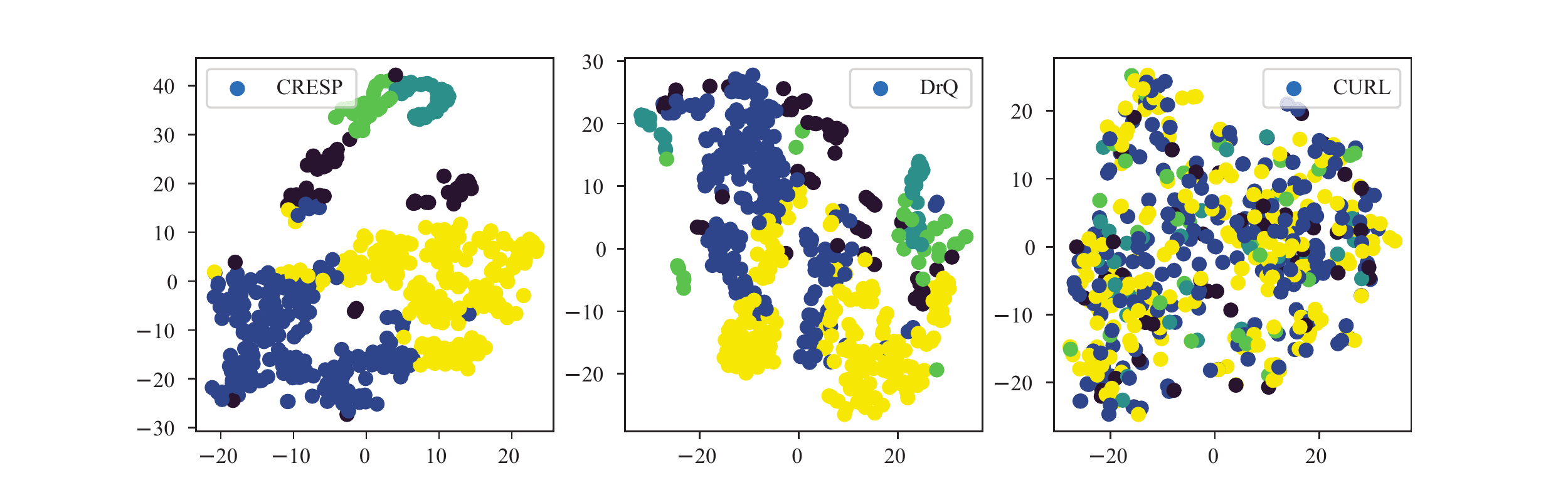}
	\vskip -0.1in
	\caption{t-SNE visualization of latent representations learned by CRESP (left), DrQ (center), and CURL (right) in cartpole-swingup with dynamic backgrounds.}
	\Description{t-SNE of latent spaces learned by CRESP and Drq.}
	\label{fig-tsne-cs}
	\vskip -0.1in
\end{figure}

\begin{figure}[h]
	\centering
	\includegraphics[width=2.4cm]{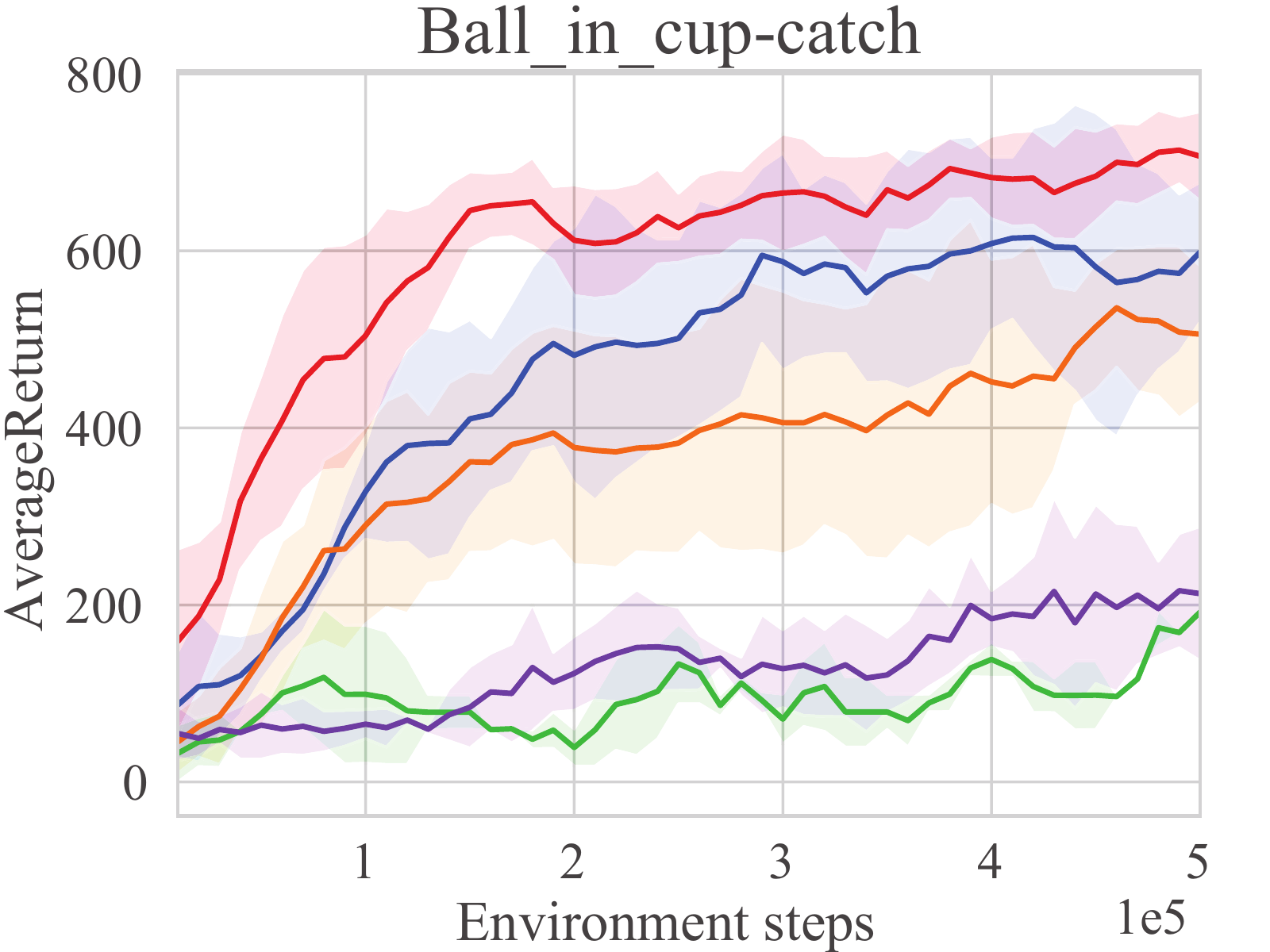}
	\hskip 0.15in
	\includegraphics[width=2.4cm]{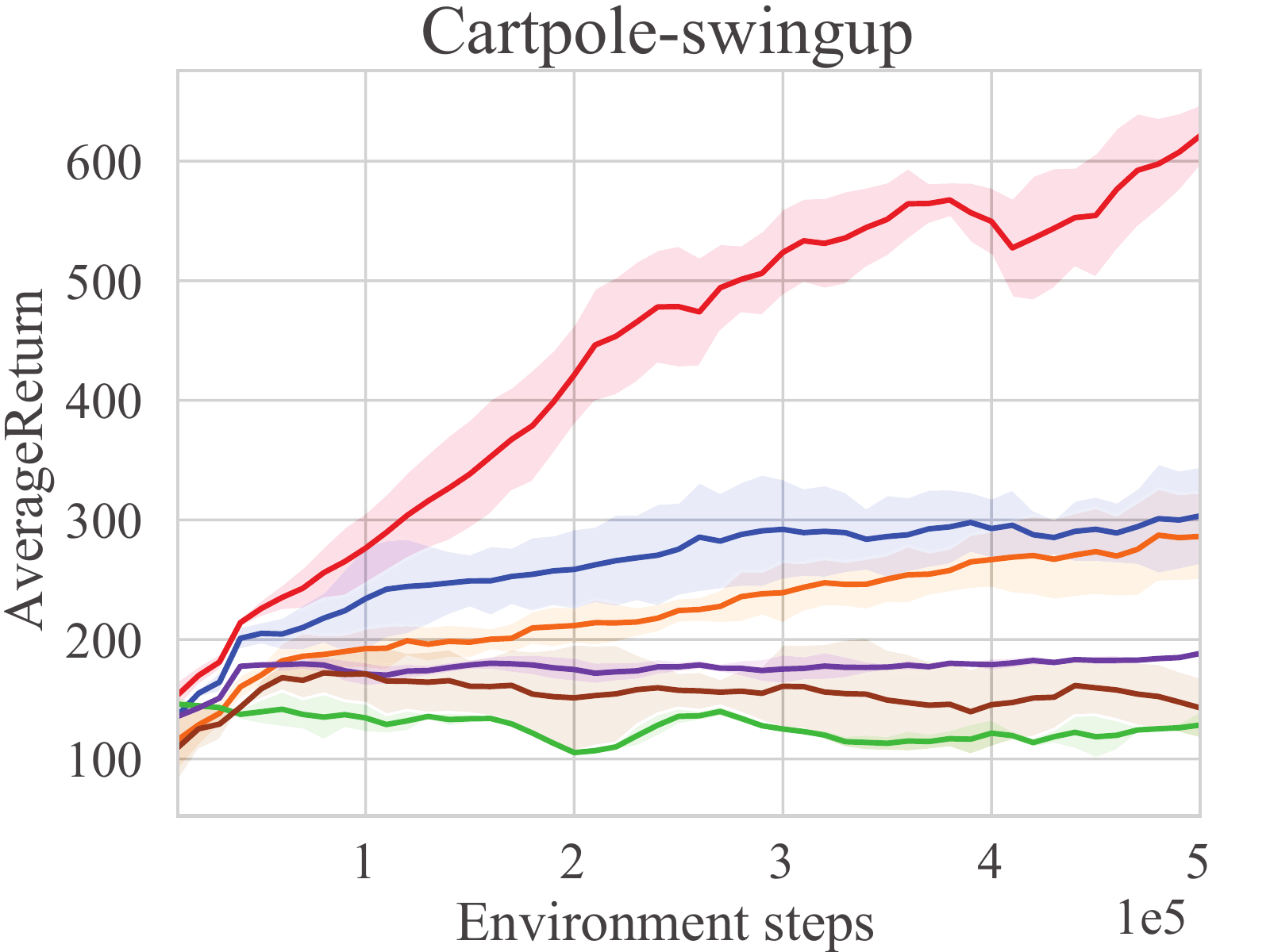}
	\hskip 0.15in
	\includegraphics[width=2.4cm]{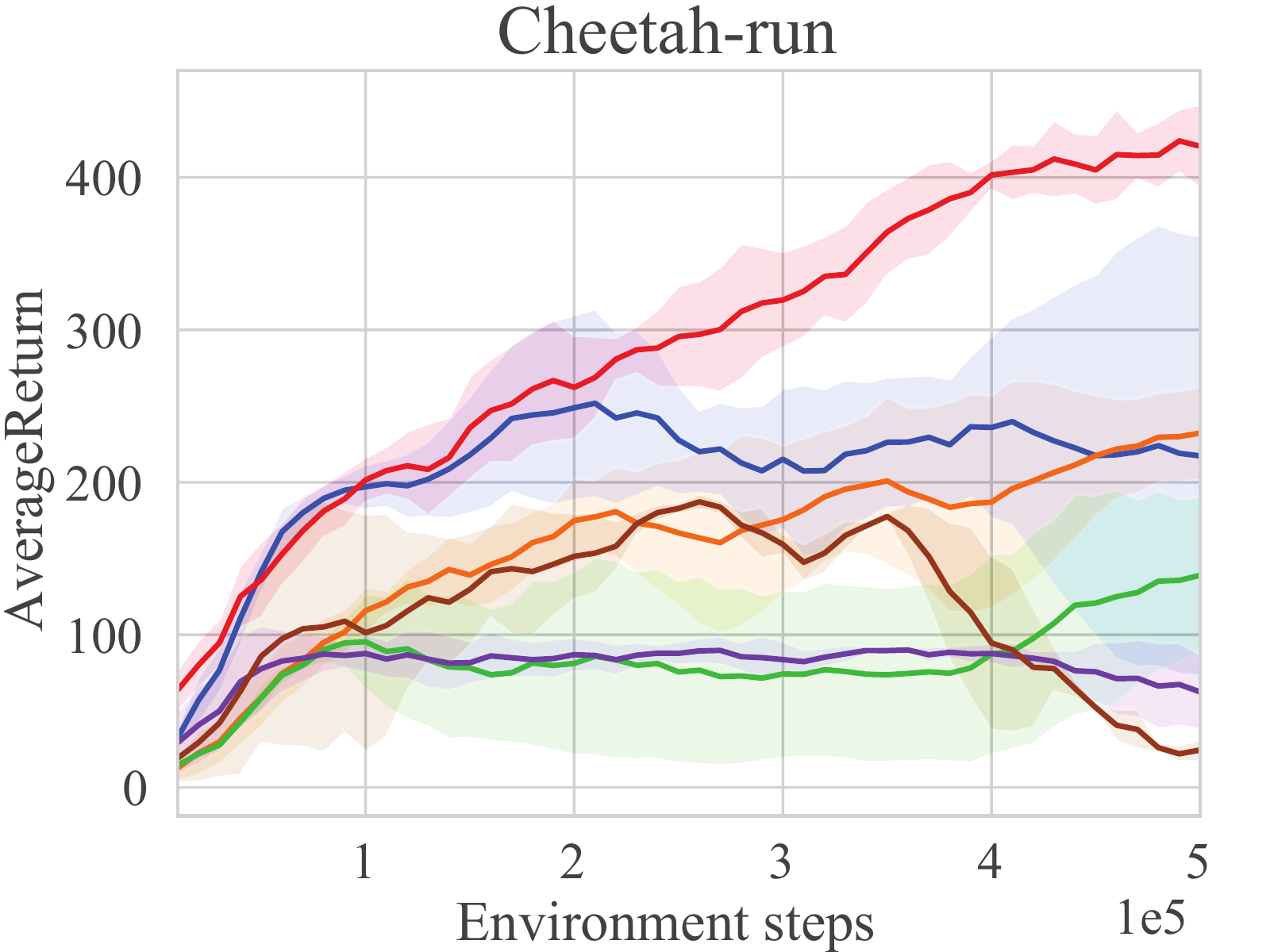}
	
	\includegraphics[width=2.4cm]{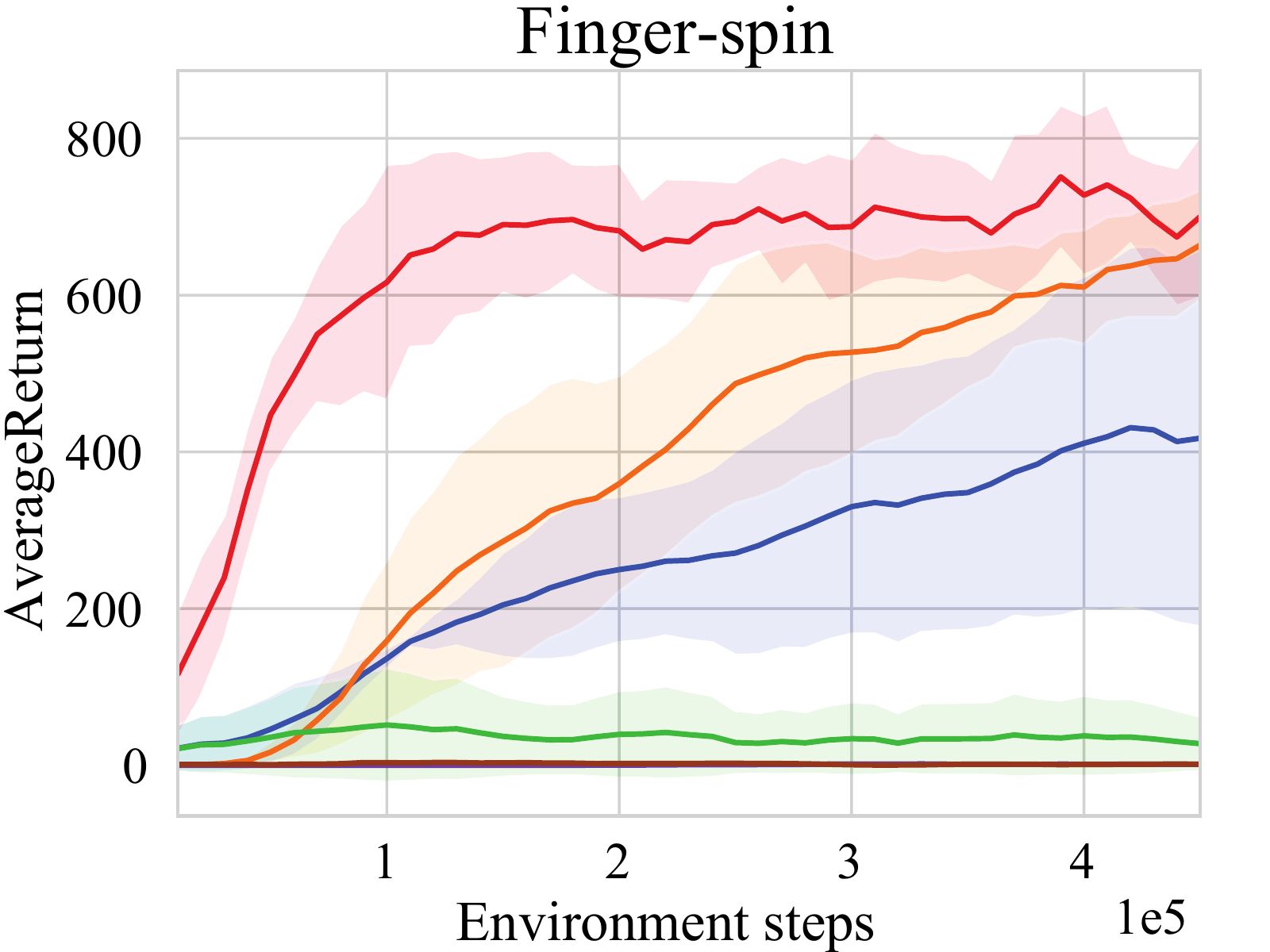}
	\hskip 0.15in
	\includegraphics[width=2.4cm]{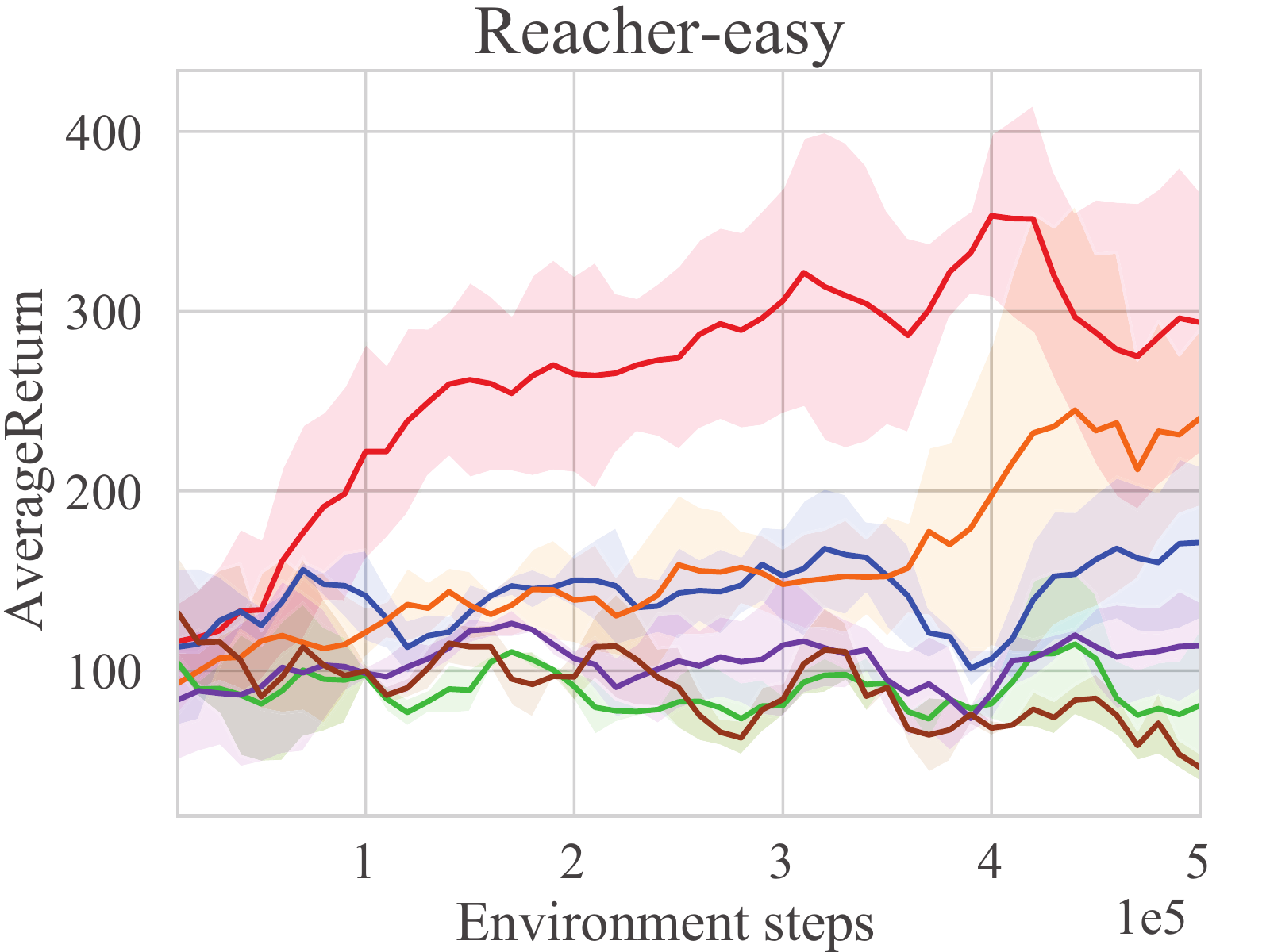}
	\hskip 0.15in
	\includegraphics[width=2.4cm]{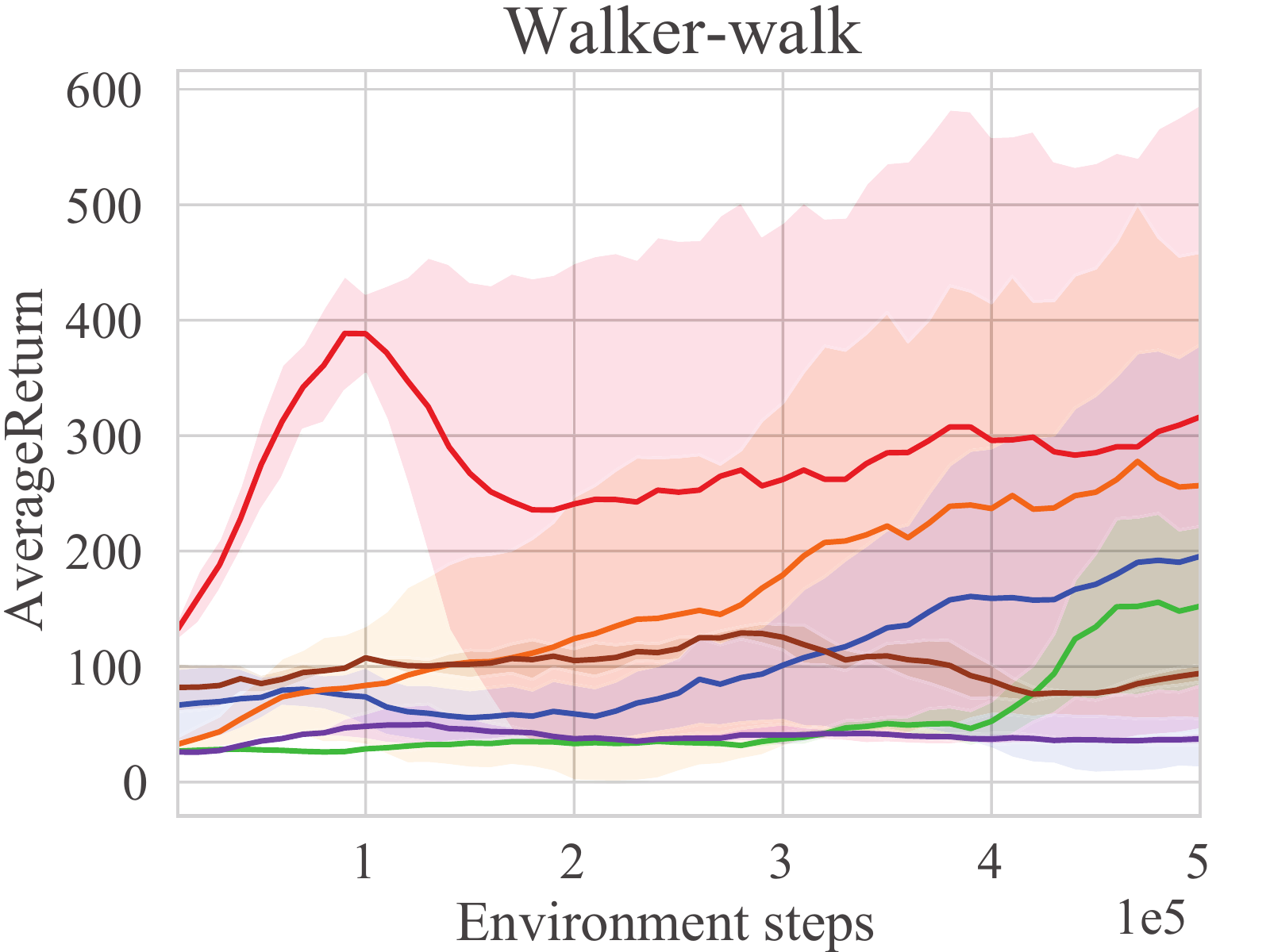}
	
	\includegraphics[width=7.4cm]{b-legend.pdf}
	\vskip -0.1in
	\caption{Learning curves of six methods on six tasks with dynamic color distractions for 500K environment steps.}
	\Description{The performances of different algorithms on unseen environments with dynamic color distractions.}
	\vskip -0.1in
	\label{fig-result-dc}
\end{figure}

\paragraph{Additional Curves}
In Figure~\ref{fig-result-dc} we show learning curves under the default settings on 6 different environments from DCS with dynamic color distractions.

\paragraph{Additional Visualizations}
In addition to Figure~\ref{fig-tsne-cs}, we also visualize the representations of CRESP and DrQ in cartpole-swingup task via t-SNE. We leverage 500 observations from 2 environments with different dynamic backgrounds. All labels are generated by KMeans with the original states as inputs.

\begin{table}
	\caption{Performance comparison with 3 seeds on cartpole-swingup with dynamic backgrounds at 500K steps. $T$ is the reward length for the ablation study.}
	\vskip -0.1in
	\label{table-sum}
	\centering
	\begin{tabular}{c|cccc}
		\toprule
		R Length & RP & RP-Sum & CRESP & CRESP-Sum \\
		\midrule
		$T=1$  & $\textbf{625} \pm \textbf{85}$ & $\textbf{625} \pm \textbf{85}$ & $616 \pm 155$ & $616 \pm 155$\\
		$T=3$  & $645 \pm 55$ & $631 \pm 35$ & $\textbf{666} \pm \textbf{42}$ & $610 \pm 88$\\
		$T=5$  & $575 \pm 111$ & $610 \pm 85$ & $\textbf{687} \pm \textbf{50}$ & $629 \pm 61$\\
		$T=7$  & $654 \pm 97$ & $599 \pm 136$ & $\textbf{667} \pm \textbf{43}$ & $639 \pm 61$\\
		\midrule
		Average  & $625 \pm 94$ & $613 \pm 96$ & $\textbf{658} \pm \textbf{98}$ & $623 \pm 100$\\
		\bottomrule
	\end{tabular}
	\vskip -0.1in
\end{table}

\subsubsection{Reward Sequence Distributions vs Distributions of Sum of Reward Sequences}
\label{app-as}
In Section~\ref{subsec-4.1}, we introduce the notion of reward sequence distributions (RSDs) to determine task relevance. We can also identify the task relevance---learning $T$-level representations---via the distributions of sum of reward sequences. Therefore, we compare the performance of learning reward sequence distributions (CRESP) with that of learning distributions of sum of reward sequences (CRESP-Sum). 
Moreover, we evaluate the method to estimate the expected sums of reward sequences (RP-Sum). We adopt the same hyperparameters and ablate the reward length $T$. 

In Table~\ref{table-sum}, we boldface the results that have highest means. The average performances of different reward lengths of RP-Sum are lower than that of RP. This result shows that the high-dimensional targets can provide more helpful information than one-dimensional targets, which is similar to the effectiveness of knowledge distillation by a soft target distribution.
Moreover, CRESP has higher average performance than CRESP-Sum, outperforming RP and RP-Sum, which empirically demonstrates that learning distributions provides benefits for representation learning in the deep RL setting.

\subsection{Code}
We implement all of our codes in Python version 3.8 and make the code available online~\footnote{\url{https://github.com/MIRALab-USTC/RL-CRESP}}.
We used NVIDIA GeForce RTX 2080 Ti GPUs for all experiments. Each trials of our method was trained for 20 hours on average.

\end{document}